\documentclass{ws-m3as}

\usepackage{amsfonts,bm}
\usepackage{graphicx}
\usepackage{epstopdf}
\usepackage{algorithmic}
\usepackage{cases,bm}
\usepackage{subfigure,multirow,lipsum}
\usepackage{mathtools}
\usepackage{amsmath,amssymb,url}
\usepackage{xcolor}

\makeatletter
\@addtoreset{Equation}{Section}
\makeatother
\ifpdf
\DeclareGraphicsExtensions{.eps,.pdf,.png,.jpg}
\else
\DeclareGraphicsExtensions{.eps}
\fi
\newtheorem{assumption}[theorem]{Assumption}

\definecolor{myColor}{rgb}{0.1, 0.2, 0.6}

\usepackage{stmaryrd}
\usepackage{hvdashln} 
	\usepackage{color}
	\usepackage{tikz,pgfplots,pgf}
	\usetikzlibrary{matrix,shapes,arrows,positioning}
	\usepackage{etoolbox} 
	\usetikzlibrary{shapes}
	\tikzset{box/.style ={
			rectangle, 
			rounded corners =5pt, 
			minimum width =50pt, 
			minimum height =20pt, 
			inner sep=5pt, 
			draw=blue} 
}

\tikzset{zbox/.style ={
		rectangle, 
		minimum width =50pt, 
		minimum height =20pt, 
		inner sep=5pt, 
		draw=black} 
}

\tikzset{ball/.style ={
	circle, 
	minimum width =20pt, 
	minimum height =20pt, 
	inner sep=0.1pt, 
	draw=blue} 
}

\tikzset{global scale/.style={
scale=#1,
every node/.append style={scale=#1}
}
}

\definecolor{darkcerulean}{rgb}{0.1, 0.1, 0.7}
\definecolor{brickred}{rgb}{0.8, 0.25, 0.33}
\definecolor{brightmaroon}{rgb}{0.65, 0.16, 0.16}
\usepackage{hyperref}
\hypersetup{
	colorlinks,
	citecolor=myColor,
	filecolor=darkcerulean,
	linkcolor=brightmaroon,
	urlcolor=darkcerulean
}

\begin{document}


%
%
\title{\Large Deep Neural ODE Operator Networks for PDEs}

\author{Ziqian Li}

\address{School of Mathematics, Jilin University, 2699 Qianjin Street, Changchun, 130012, Jilin, China.\\
	 Chair for Dynamics, Control, Machine Learning, and Numerics (Alexander von Humboldt Professorship), Department of Mathematics, Friedrich–Alexander-Universit\"at Erlangen–N\"urnberg, 91058 Erlangen, Germany.\\
\href{zqli23@mails.jlu.edu.cn}{zqli23@mails.jlu.edu.cn}}

\author{Kang Liu}

\address{Institut de Mathematiques de Bourgogne, Université Bourgogne Europe, CNRS, 21000 Dijon, France.\\
\href{kang.liu@u-bourgogne.fr}{kang.liu@u-bourgogne.fr}}

\author{Yongcun Song}
\address{Division of Mathematical Sciences, School of Physical and Mathematical Sciences,
	Nanyang Technological University,
	21 Nanyang Link, 637371, Singapore.\\ \href{yongcun.song@ntu.edu.sg}{yongcun.song@ntu.edu.sg}}
	
\author{Hangrui Yue}
\address{School of Mathematical Sciences, Nankai University, Tianjin 300071, China. \\ \href{yuehangrui@gmail.com}{yuehangrui@gmail.com}}

\author{Enrique Zuazua}
\address{Chair for Dynamics, Control, Machine Learning, and Numerics (Alexander von Humboldt Professorship), Department of Mathematics, Friedrich–Alexander-Universit\"at Erlangen–N\"urnberg, 91058 Erlangen, Germany. \\
	Departamento de Matem\'aticas, Universidad Aut\'onoma de Madrid, 28049 Madrid, Spain. \\
	Chair of Computational Mathematics, Fundaci\'on Deusto. Av. de las Universidades, 24, 48007 Bilbao, Basque Country, Spain. \\  \href{enrique.zuazua@fau.de}{enrique.zuazua@fau.de}}

\maketitle


\begin{abstract}
Operator learning has emerged as a promising paradigm for developing efficient surrogate models to solve partial differential equations (PDEs). However, existing approaches often overlook the domain knowledge inherent in the underlying PDEs and hence suffer from challenges in capturing temporal dynamics and generalization issues beyond training time frames. This paper introduces a deep neural ordinary differential equation (ODE) operator network framework, termed NODE-ONet, to alleviate these limitations. The framework adopts an encoder-decoder architecture comprising three core components: an encoder that spatially discretizes input functions, a neural ODE capturing latent temporal dynamics, and a decoder reconstructing solutions in physical spaces. Theoretically, error analysis for the encoder-decoder architecture is investigated. Computationally, we propose novel physics-encoded neural ODEs to incorporate PDE-specific physical properties. Such well-designed neural ODEs significantly reduce the framework's complexity while enhancing numerical efficiency, robustness, applicability, and generalization capacity. Numerical experiments on nonlinear diffusion-reaction and Navier-Stokes equations demonstrate high accuracy, computational efficiency, and prediction capabilities beyond training time frames. Additionally, the framework’s flexibility to accommodate diverse encoders/decoders and its ability to generalize across related PDE families further underscore its potential as a scalable, physics-encoded tool for scientific machine learning.
\end{abstract}

\keywords{Operator learning, neural ordinary differential equations, partial differential equations, error analysis, neural networks, scientific machine learning.}

\ccode{AMS Subject Classification: 65M99,  
47-08,  	
68T07,  	
65D15  	
}

\section{Introduction}
\subsection{Background}
Partial Differential Equations (PDEs) serve as fundamental tools for modeling systems across physics, engineering, biology, and economics, where quantities such as temperature, pressure, wave amplitude, or population density evolve in space and time. Analytical solutions to PDEs are often intractable, particularly for nonlinear or high-dimensional problems. Consequently, numerical methods have become essential for solving PDEs approximately.  

Traditional numerical methods for PDEs, such as finite difference methods (FDM), finite element methods (FEM), finite volume methods, and spectral methods, discretize continuous temporal-spatial domains into computational grids, transforming PDEs into solvable algebraic systems.
These methods are grounded in rigorous mathematical theories and can produce highly accurate and interpretable solutions, especially for linear PDEs and well-posed problems. Moreover, these methods are efficient for low-to-moderate dimensionality (no more than 3) and the corresponding linear system solvers and preconditioners are highly optimized. Hence, these traditional numerical methods are indispensable for their precision and reliability in diverse PDE applications. However, traditional numerical methods are still struggling to address PDEs in high-dimensional spaces and complex domains. 

Many problems in science and engineering, such as inverse problems and optimal control, require solving PDEs repeatedly for different parameters \cite{song2023accelerated,song2024operator,tanyu2023deep,wang2021fast}. Traditional numerical solvers for these problems rely on mesh-based discretization, which generates large-scale algebraic systems. Solving these systems repeatedly is computationally expensive, rendering traditional approaches inefficient for many-query scenarios. To alleviate this issue, we propose a novel operator learning framework for solving PDEs. 

Operator learning follows an offline training and online inference paradigm.  As analyzed in \cite{chen2023deep}, operator learning lessens the curse of dimensionality for solving PDEs. However, we should mention that operator learning usually requires training data generated by traditional numerical PDE solvers, and thus cannot fully escape the curse of dimensionality.  Nevertheless, operator learning involves only a single, offline training phase for the neural network and during inference, obtaining a solution for a new input function requires only a forward pass of the neural network. This enables the construction of highly effective emulators capable of real-time predictions for critical applications such as freeway traffic control, weather forecasting, and digital twins, see e.g., \citeup{Azizzadenesheli2024neural,kobayashi2024improved,kobayashi2024deep,liu2024deep,lv2025neural,pathak2022fourcastnet} and the references therein.

\subsection{Methodology and contributions}
To impose our ideas clearly, we consider a class of  PDEs  modeled by
\begin{equation}\label{e_pde}
	\left\{
	\begin{aligned}
		&\partial_t u(t,x)+\mathcal{L}[a](u)(t,x)=f(t,x) &\forall (t,x)\in [0,T]\times \Omega,\\
		&u(0,x)=u_0(x) & \forall x\in \Omega,\\
		&\mathcal{B}u(t,x)=u_b(t,x) &\forall (t,x)\in [0,T]\times \partial\Omega.
	\end{aligned}
	\right.
\end{equation}
Above, $T>0$, $\Omega\subset \mathbb{R}^d$ is a bounded domain with boundary $\partial\Omega$, $\mathcal{L}$ is a differential operator with $a: [0,T]\times 
\Omega\to \mathbb{R}$ as the underlying parameter (e.g., $\mathcal{L}[a](u)(t,x)=-\nabla\cdot(a(t,x))\nabla u(t,x))$, $f: [0,T]\times 
\Omega\to \mathbb{R}$ is the source term, $u_0: \Omega\to \mathbb{R}$ is the initial value, $\mathcal{B}$ denotes a
boundary conditions operator that enforces any Dirichlet, Neumann, Robin, or periodic boundary conditions, and $u_b: [0,T]\times 
\partial\Omega\to \mathbb{R}$ is the boundary value. 

Let $v \subset \{f, a, u_0, u_b\}$ denote a collection of parameters in \eqref{e_pde}.  We assume that, for any admissible $v$, the system \eqref{e_pde} has a unique classical solution $u$ in the appropriate function space. 
The goal of operator learning is to approximate the solution operator $\Psi^\dagger$, which maps $v$ to the solution $u$, by a neural-network--based functional $\Psi_{\theta}$ with trainable parameters $\theta$:
\begin{equation*}
	\Psi_{\theta} \approx \Psi^\dagger : v \mapsto u.
\end{equation*}
In this article, we introduce $\Psi_{\theta}$ within a deep Neural Ordinary Differential Equation Operator Network (NODE-ONet) framework.
\subsubsection{The NODE-ONet framework}
The NODE-ONet employs an encoder--decoder architecture, with its error analysis presented in Section~\ref{se:error_analysis}. 
In practice, the NODE-ONet consists of the following three components (a concrete example is provided in Section \ref{se:example}, and a complete version is presented in Section~\ref{se_NODE-ONet}):
\begin{enumerate}
	\item \textbf{Encoding:} The input parameter set $v$ is embedded into a latent space using a spatial discretization scheme (e.g., pointwise evaluation on grids, expansion in a finite element or Fourier basis).  
	The resulting encoded system of \eqref{e_pde} is represented by a reduced number of state variables, referred to as \emph{latent variables}, whose dynamics evolve in time.
	\item \textbf{NODE surrogate:} We develop physics-encoded NODEs, which employ \emph{explicitly time-dependent} parameters (e.g., with polynomial dependence on $t$) to approximate the dynamics of the latent variables, thereby substantially reducing the complexity of vanilla NODEs \cite{chen2018neural}.  
	Furthermore, the architecture of physics-encoded NODEs is designed to encode structural properties of the underlying PDEs, such as the nonlinear dependence of $a$ and $u$, as well as the additive relationship between $f$ and $u$ in \eqref{e_pde}.  
	The effects of other known parameters (i.e., $\{f, a, u_0, u_b\}\setminus v$ in \eqref{e_pde}) are also incorporated into the NODE design.
	\item \textbf{Decoding:} After the latent dynamics are learned, the PDE solution $u$ is reconstructed from the NODE outputs by a decoder that depends only on the spatial domain.
\end{enumerate}

Note that the temporal and spatial variables are treated separately in the NODE-ONet. Such separation is in alignment with traditional numerical methods for time-dependent PDEs, which often employ a time-sequential method and solve the problems step by step instead of the entire temporal-spatial domain. As a result, the trained space-dependent decoder can be generalized to other operator learning tasks for PDEs with similar structures, as demonstrated in Section \ref{se:num_DR}. Moreover, it is remarkable that although the NODE-ONets are designed for evolution PDEs, they can be readily applied to stationary cases, see Remark \ref{rem:example}, Corollary \ref{cor:example}, and Section \ref{se: architecture}.

\subsubsection{Contributions}
We introduce NODE-ONets, a new framework for operator learning in PDEs, with theoretical and practical advancements.
Theoretically, we establish error estimates in Theorem~\ref{thm:main} for the general encoder--decoder architecture, which can serve as a foundation for proving the convergence of NODE-ONets on a case-by-case basis in future work.  Practically, we develop physics-encoded NODEs, which are central to the efficiency, robustness, and broad applicability of the proposed NODE-ONets, as will be demonstrated in Section~\ref{se:num}.

The physics-encoded design enables two critical capabilities. First, it allows the model to effectively capture dynamic patterns of the underlying PDEs and predict system evolution beyond the training time frame. Second, it seamlessly scales to multi-input functions without increasing the neural network complexity. This scalability makes NODE-ONets adaptable to diverse input configurations while maintaining computational efficiency. 

{It is noteworthy that the physics-encoded NODEs are designed by leveraging the specific mathematical structure of underlying PDEs, while the neural networks are trained offline on PDE-generated data. Hence, the physics-encoded NODE-ONets demonstrate a promising synergy between PDE-based domain knowledge and data-driven paradigms. Such a hybrid approach benefits from the interpretability and reliability of PDEs and the flexibility and generalization of neural networks.}

\subsection{Related work and motivation}
Some deep learning methods have been recently proposed for solving PDEs, as evidenced by~\cite{e2018deep,li2021fourier,lu2021learning,raissi2019physics,sirignano2018dgm} and the references therein.
Thanks to the powerful representation capabilities~\cite{cybenko1989approximation,hornik1989multilayer,kidger2020universal} and generalization abilities~\cite{kawaguchi2017generalization,neyshabur2017exploring} of deep neural networks (DNNs), the above deep learning methods enhance the feasibility of tackling complex, multi-scale PDE systems and have found widespread applications in various scientific and engineering fields, see e.g., \cite{cuomo2022scientific,gao2025prox,karniadakis2021physics,lai2025hard,song2024admm,tanyu2023deep} and the references therein.

Compared with traditional numerical methods, deep learning methods are typically mesh-free, easy to implement, and flexible in solving various PDEs, especially for high-dimensional problems or those with complex geometries.
Some deep learning methods, such as the deep Ritz method \cite{e2018deep}, the deep Galerkin method \cite{sirignano2018dgm}, and physics-informed neural networks (PINNs) \cite{lu2021physics,raissi2019physics}, approximate the solution of a given PDE by DNNs, and the PDE solution can be obtained by training the DNNs.
Despite that these methods have shown promising results in diverse applications, each of them is tailored for a specific instance of PDEs. It is thus necessary to train a new DNN given a different PDE parameter (e.g., initial/boundary value, source term, or coefficients), which may be computationally costly and thus challenging to generate real-time predictions for varying input data. To alleviate this issue, some operator learning methods have been recently proposed in the literature, see e.g., \cite{bhattacharya2021model,kovachki2023neural,li2021fourier,lu2021learning}

\subsubsection{Operator learning for PDEs}
Operator learning applies a DNN to approximate the
solution operator of a PDE, which maps from a PDE parameter to the solution.
Once a neural solution operator is learned, we obtain a neural surrogate model and only require a forward pass of the DNN to solve a PDE. Representative operator learning methods for solving PDEs include the deep operator networks (DeepONets) \cite{lu2021learning}, the MIONet \cite{jin2022mionet}, the physics-informed DeepONets \cite{wang2021learning}, the Fourier neural operator (FNO) \cite{li2021fourier}, the graph neural operator \cite{li2020neural}, the random feature model \cite{nelsen2021random}, the PCA-Net \cite{bhattacharya2021model}, the Laplace neural operator  \cite{cao2023lno}, and the in-context operator network \cite{yang2023context}. 

Although the traditional numerical PDE solvers may still be competitive in pursuit of solutions with high accuracy, it is well known that operator learning approaches could also achieve satisfactory accuracy of solutions while additionally gaining significance in numerical efficiency as well as generalization ability. Hence, operator learning approaches are being widely used to construct effective surrogates for PDEs and are computationally attractive for problems that require repetitive yet expensive simulations, see e.g., \cite{Azizzadenesheli2024neural,hwang2022solving,kobayashi2024improved,kobayashi2024deep,liu2024deep,lv2025neural,pathak2022fourcastnet,song2023accelerated,song2024operator,wang2021fast}.

Among the above operator learning methods, the DeepONets, which adopt an encoder-decoder architecture and are motivated by the universal approximation theorems for operators \cite{chen1995universal,lu2021learning}, have demonstrated good performance in diverse applications, such as electroconvection \cite{cai2021deepm}, multiscale bubble growth dynamics \cite{lin2021operator}, and aortic dissection \cite{yin2022simulating}.  More applications, theoretical analysis, and numerical study of DeepONets can be referred to \cite{faroughi2024physics,hao2022physics,lu2022comprehensive}.  Several variants of DeepONets have also been developed for learning PDE solution operators in different settings, such as the Bayesian DeepONet \cite{garg2022variational}, the DeepONet with
proper orthogonal decomposition (POD-DeepONet) \cite{lu2022comprehensive}, and the physics-informed DeepONets \cite{wang2021learning}. 

Note that all these DeepONets treat temporal and spatial variables together and train the neural networks over the entire temporal-spatial domain at once. This may lead to neural networks being hard to train and deteriorate the numerical accuracy, see e.g., \cite{krishnapriyan2021characterizing,wang2024respecting}. Despite their broad applicability, DeepONets always utilize generic neural networks (e.g., fully connected neural networks (FCNNs), convolutional neural networks (CNNs), and recurrent neural networks) and overlook domain-specific knowledge inherent to the PDEs, such as their intrinsic structure or the  effect of specific PDE parameters. This oversight may compromise their computational efficiency in learning PDE solution operators. In particular, DeepONets exhibit limitations in capturing the evolution dynamics of time-dependent  PDEs, which restricts their predictive performance beyond the training temporal domain. 

Furthermore, as validated in \cite{jin2022mionet}, vanilla DeepONets \cite{lu2021learning} are not sufficiently accurate when approximating PDE solution operators involving multiple input functions. To address this issue, the MIONet was proposed in \cite{jin2022mionet}, which enhances the numerical accuracy for learning operators with multi-input functions. However, the model complexity of MIONet increases as the number of input functions scales, and the MIONet still suffers from the prediction issue beyond training time frames.

These limitations motivate architectures that (i) treat time as a continuous variable rather than fixed training grids, (ii) encode the dynamical structure of PDEs explicitly, and (iii) remain stable when extrapolating beyond the training horizon. A natural candidate in this direction is the framework of NODEs.

\subsubsection{Neural ODEs}\label{se:intro_penode}
To handle time-dependent systems, the authors in \cite{chen2018neural} introduced the concept of NODEs, a continuous-time limit of deep residual networks (ResNets) \cite{he2016deep} that merges deep learning with dynamical-systems theory. Classical ResNets can be viewed as a forward-Euler discretization of a NODE. 
\color{black}
Formally, a vanilla NODE models the dynamics of a state trajectory $\bm{x}(t):[0,T]\to\mathbb{R}^d$ via an ODE parameterized by a neural network:
\begin{equation}
	\label{eq:trad-NODE}
	\begin{dcases}
		\dot{\bm{x}} =\sum_{i = 1}^P W_i(t)\odot\bm{\sigma}(A_i(t)\bm{x}+B_i(t)), \\
		\bm{x}(0) = x_0.
	\end{dcases}
\end{equation}
Here, $A_i \in L^\infty([0,T];\mathbb{R}^{d \times d}), W_i \in L^\infty([0,T];\mathbb{R}^d),$ and $B_i \in L^\infty([0,T];\mathbb{R}^d)$ for $i = 1, \ldots, P$ are
trainable parameters, $\bm{\sigma}: \mathbb{R}^d \to \mathbb{R}^d$ is an activation function, and $\odot$ stands for the Hadamard product. 
The continuous-time modeling capability of the vanilla NODEs offers significant advantages for applications requiring smooth interpolations or handling irregularly sampled data, such as time series modeling \cite{chen2018neural} and classification tasks \cite{ruiz2023neural}. In particular, the application of NODEs to solve PDEs has been studied in \cite{nair2025understanding,regazzoni2024learning}.

{However, vanilla NODEs \cite{chen2018neural} suffer from several critical limitations when applied to operator learning.}
First, the parameters therein are all implicitly time-dependent, which introduces substantial computational complexity when training with finer temporal resolutions. Furthermore, the implicit temporal dependence of these parameters restricts generalization. 
To be concrete, we can only get their values at $\{t_k\}_{k=1}^{N_t}$ that are used in the training process, rendering predictions at $\hat{t}\notin\{t_k\}_{k=1}^{N_t}$ unreliable or unavailable.  We refer to \cite{li2024universal} for related discussions. 
Second, vanilla NODEs are largely restricted to learning solution operators conditioned solely on the initial condition (see \cite{nair2025understanding,regazzoni2024learning}). This limits their applicability to broader operator-learning settings in which the inputs include, for example, spatially varying diffusion coefficients, source terms, boundary data, or other PDE parameters.
Finally, standard implementations of NODEs, which rely on conventional architectures such as FCNNs or CNNs, neglect PDE-specific knowledge.

Therefore, the main goal of this article is to develop a physics-encoded NODE–based operator-learning framework that learns the operator mapping from PDE parameters to the corresponding solution. By encoding PDE structure into the NODE architecture and incorporating numerical solutions into the training data, our method enables an efficient offline--online workflow: after a one-time training cost, the surrogate provides fast, high-fidelity predictions for new parameter queries, yielding high-performance numerical resolution of PDEs.

\subsection{Organization}
The remainder of this paper is organized as follows. In Section \ref{se:general edn}, we introduce the general encoder-decoder networks and investigate the corresponding error analysis. Then, we present the generic architecture of the NODE-ONet framework and discuss the training methodology in Section \ref{se: architecture}. In Section \ref{se:pe_node}, we elaborate on the design of physics-encoded NODEs and present the resulting NODE-ONets. In Section \ref{se:num}, comprehensive numerical results are presented to validate the effectiveness, efficiency, and flexibility of the NODE-ONets in different contexts. Finally, some conclusions and research perspectives are given in Section \ref{se:conclusion}.

\section{General Encoder-Decoder Networks and Error Analysis}\label{se:general edn}

In this section, we first review the encoder-decoder networks \cite{bhattacharya2021model,kovachki2024operator,lu2021learning}, a widely used architecture in operator learning. We then present a general error analysis framework for such networks. This analysis framework serves as the foundation for both the design and theoretical analysis of the proposed deep NODE-ONets, which are developed in subsequent sections for solving PDEs.

\subsection{The architecture of encoder-decoder networks}\label{sec:general_encoder-decoder}
Encoder-decoder networks aim to approximate the operator $\Psi^\dagger: \mathcal{V} \to \mathcal{U}$ with infinite-dimensional input and output spaces $\mathcal{V}$ and $\mathcal{U}$. For this purpose, we first introduce two suitable latent spaces, denoted by $\mathcal{V}_h$ and $\mathcal{U}_h$, respectively, which generally possess a simpler structure than the original spaces $\mathcal{V}$ and $\mathcal{U}$, often characterized by finite dimensionality.
Then, we select an encoder \(E_{\mathcal{V}} \) and a decoder \( D_{\mathcal{U}}\):
\[
E_{\mathcal{V}}\colon \mathcal{V} \to \mathcal{V}_h\quad \text{and} \quad D_{\mathcal{U}}\colon \mathcal{U}_h\to \mathcal{U},
\]
either fixed \emph{a priori} or parameterized by neural networks. Next, we design a decoder $D_{\mathcal{V}}$ and an encoder $E_{\mathcal{U}}$:
\[
D_{\mathcal{V}}\colon \mathcal{V}_h \to \mathcal{V} \quad\text{and}\quad E_{\mathcal{U}}\colon \mathcal{U} \to \mathcal{U}_h,
\]
so that the compositions 
\[
D_{\mathcal{U}}\circ E_{\mathcal{U}} \quad \text{and} \quad D_{\mathcal{V}}\circ E_{\mathcal{V}}
\]
approximate the identity mappings in $\mathcal{U}$ and $\mathcal{V}$ (see Assumption~\ref{ass1}(1)-(2) for the linear case).

These encoder/decoder pairs in turn imply an encoding of the underlying infinite-dimensional operator $\Psi^\dagger$, resulting in a function between the latent spaces $ \mathcal{V}_h$ and $ \mathcal{U}_h$:
\[
\psi: \mathcal{V}_h \to \mathcal{U}_h, \quad \psi(\zeta) = E_\mathcal{U} \circ \Psi^\dagger \circ D_\mathcal{V}(\zeta),\quad \forall \zeta\in  \mathcal{V}_h,
\]
as depicted in the right-hand-side of Figure \ref{fig:structure}. Then, formally, we have
$$
D_\mathcal{U}\circ\psi\circ E_{\mathcal{V}}\approx \Psi^\dagger.
$$

Using a neural network $\psi_{\theta}\colon \mathcal{V}_h \to \mathcal{U}_h$ to approximate $\psi$, we obtain an encoder-decoder network $\Psi_{\theta} \colon \mathcal{V}\to\mathcal{U}$ in the form of
\begin{equation}\label{eq:Psi_theta}
	\Psi_{\theta} \coloneqq D_\mathcal{U}\circ\psi_{\theta}\circ E_{\mathcal{V}},
\end{equation}
to approximate the operator $\Psi^\dagger$, as illustrated in the left-hand-side of Figure \ref{fig:structure}.
Representative encoder-decoder network architectures include the PCA-Net \cite{bhattacharya2021model}, the Integral Autoencoder Network \cite{ong2022integral}, the DeepONets \cite{lu2021learning}, the MIONet \cite{jin2022mionet}, and their variants \cite{choi2024spectral,hua2023basis,prasthofer2022variable,wang2021learning,zhang2023belnet}, just to name a few. 

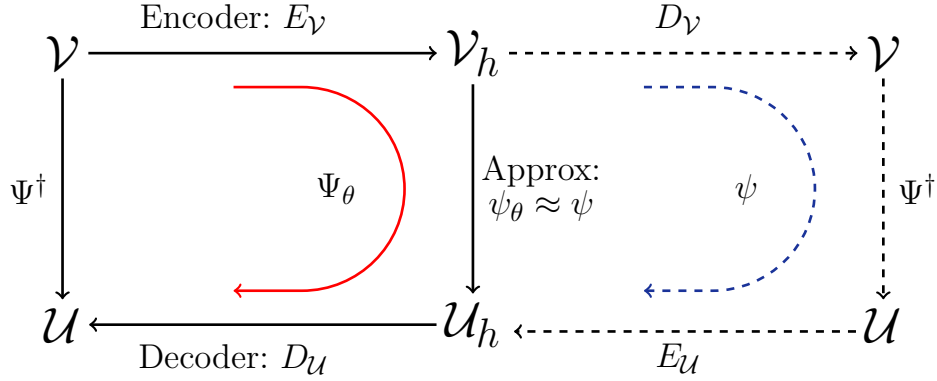
\begin{figure}[htpb]
	\label{fig:structure}
	\centering
	\begin{tikzpicture}[global scale =0.9]
		\node at(2,6)   (1) [font=\fontsize{20}{20}\selectfont]{$\mathcal{V}$};
		\node at(8,6)   (2) [font=\fontsize{20}{20}\selectfont]{$\mathcal{V}_h$};
		\node at(14,6)  (3) [font=\fontsize{20}{20}\selectfont]{$\mathcal{V}$};
		\node at(14,2)   (4) [font=\fontsize{20}{20}\selectfont]{$\mathcal{U}$};
		\node at(8,2)   (5) [font=\fontsize{20}{20}\selectfont]{$\mathcal{U}_h$};
		\node at(2,2)   (6) [font=\fontsize{20}{20}\selectfont]{$\mathcal{U}$};
		\node at(1.5,4)   (7) [font=\fontsize{15}{15}\selectfont]{$\Psi^{\dagger}$};
		\node at(9,3.75)   (8) [font=\fontsize{15}{15}\selectfont]{$\psi_{\theta}\approx\psi$};
		\node at(14.5,4)   (9) [font=\fontsize{15}{15}\selectfont]{$\Psi^{\dagger}$};
		\node at(4.5,6.5)   (10) [font=\fontsize{15}{15}\selectfont]{Encoder: $E_{\mathcal{V}}$};
		\node at(4.5,1.5)   (11) [font=\fontsize{15}{15}\selectfont]{Decoder: $D_{\mathcal{U}}$};
		\node at(11,6.5)   (12) [font=\fontsize{15}{15}\selectfont]{$D_{\mathcal{V}}$};
		\node at(11,1.5)   (14) [font=\fontsize{15}{15}\selectfont]{$E_{\mathcal{U}}$};
		\node at(9,4.25)   (15) [font=\fontsize{15}{15}\selectfont]{Approx:};
		\node at(6,4)   (16) [font=\fontsize{15}{15}\selectfont]{$\Psi_{\theta}$};
		\node at(12,4)   (17) [font=\fontsize{15}{15}\selectfont]{$\psi$};
		\draw[->][line width=1.0pt] (1)--(2);
		\draw[->][line width=1.0 pt] (1)--(6);
		\draw[->][line width=1.0 pt,dashed] (2)--(3);
		\draw[->][line width=1.0 pt] (2)--(5);
		\draw[->][line width=1.0 pt,dashed] (3)--(4);
		\draw[->][line width=1.0pt] (5)--(6);
		\draw[->][line width=1.0pt,dashed] (13.6,1.9)--(8.6,1.9);
		\draw[->, red, line width=1.0pt]  
		(4.5,5.5) -- (5.5,5.5) 
		arc[start angle=90, end angle=-90, radius=1.5] 
		-- (4.5,2.5); 
		\draw[->][draw=myColor][line width=1.0pt,dashed]  
		(10.5,5.5) -- (11.5,5.5) 
		arc[start angle=90, end angle=-90, radius=1.5] 
		-- (10.5,2.5);
	\end{tikzpicture}
	\caption{\normalsize Latent structure in maps between infinite-dimensional spaces $\mathcal{V}$ and $\mathcal{U}$.}
\end{figure}

\subsection{An illustrative example}\label{se:example}
We now demonstrate the integration of NODEs into an encoder-decoder architecture for learning PDE solution operators with a concrete example; a general framework is detailed in Section \ref{se: architecture}. Consider equation~\eqref{e_pde} in one spatial dimension ($d = 1$).  
The spatial operator takes the following form for a fixed coefficient field \( a(t,x) \) and a nonlinear function $R$:
\[
\mathcal{L}[a](u)(t,x) = -\nabla \cdot \big(a(t,x)\nabla u(t,x)\big) + R\big(u(t,x)\big).
\]  
The boundary condition \( u_b \) is fixed, while the initial condition is treated as a parameter of the PDE, given by \( v = u_0 \in \mathcal{C}(\Omega) \).  
This example aims to learn the initial-value-to-solution operator
\[
\Psi_{\mathrm{initial}}^\dagger: u_0 \mapsto u.
\]

First, we choose the latent spaces as
$\mathcal{V}_h= \mathbb{R}^{N_x}$ and $\mathcal{U}_h = \mathcal{C}([0,T];\mathbb{R}^{N_x})$ with some $N_x \in \mathbb{N}_+$.
The encoding and decoding steps are respectively realized by a uniform finite-difference grid and by a $P_1$ finite element interpolation. In particular, let $\{x_i\}_{i=1}^{N_x}$  be the uniform mesh of $\Omega$ and $\alpha_i $ the $P_1$-FEM basis centered at $x_i$. Then,
the resulting encoder-decoder architecture is given by:
\[ \textbf{Architecture}
\left\{
\begin{aligned}
	&\text{Encoder $E_{\mathcal{V}}$: } \;
	u_0 \;\mapsto\; U_{0,h} = \big(u_{0}(x_i)\big)_{i=1}^{N_x} \in \mathbb{R}^{N_x} ,
	 \\[0.6em]
	&\text{NODE surrogate $\psi_{\theta}$: } U_{0,h} \mapsto U_{\theta}, \; \text{with }
	\begin{cases}
		\dot{U}_{\theta}(t) = \mathcal{NN}_{\theta}(U_{\theta}(t),t), \\[0.4em]
		U_{\theta}(0) = U_{0,h},
	\end{cases} \\[0.6em]
	&\text{Decoder $D_{\mathcal{U}}$: } \;
	u(t,x) \;=\; \sum_{i=1}^{N_x}  \big(U_{\theta}(t)\big)_i \,\alpha_i(x).\;  \\
\end{aligned}
\right.
\]
Here, $\mathcal{NN}_{\theta}$ denotes a neural network function with parameters $\theta$ that are independent of $t$, which marks a
key distinction from the vanilla NODE formulation in \eqref{eq:trad-NODE}.  The specialized structures of $\mathcal{NN}_{\theta}$, designed to incorporate PDE information, are demonstrated in \eqref{eq:NODE_pe3} and \eqref{eq:NODE_ns}.

To train the model, we generate a dataset using an FDM and optimize the associated loss function.
Let $N_t$ denote the number of temporal discretization steps in FDM, with $t_j$ ($j=1,\ldots,N_t$) representing the discrete time nodes. Similarly, let $N_v$ denote the number of input samples, and $u_0^k$ ($k=1,\ldots,N_v$) their corresponding initial inputs.
The training setting is summarized as follows:
\[
{\small
\textbf{Training:}\;
\left\{
\begin{aligned}
	&\text{Dataset: } \begin{cases}
		\text{Features: }  U_{0,h}^k \in \mathbb{R}^{N_x}, \; \text{obtained from encoding } u_0^k;\\[0.4em]
		\text{Labels: } U_h^k \in \mathbb{R}^{N_x \times N_t}, \; \text{obtained from a FDM discretization of \eqref{e_pde};}
	\end{cases}\\[0.8em]
	&\text{Loss function: } \;
	L(\theta) = \frac{1}{N_v N_x N_t}\sum_{k=1}^{N_v}\sum_{i=1}^{N_x}\sum_{j=1}^{N_t} 
	\left| \big(U_{\theta}^k(t_j)\big)_i - U_h^k(x_i,t_j) \right|^2 \;+\; \mathcal{R}(\theta).
\end{aligned}
\right.
}
\]
Here, $U_{\theta}^k$ denotes the solution of the NN surrogate with initial condition $U_{0,h}^k$, and $\mathcal{R}(\theta)$ is a general regularization term.

The training procedure is carried out in an offline setting. Once the optimal parameter $\theta^*$ is obtained, the inference of the solution to \eqref{e_pde} for a new initial condition $u_0$ is performed by reusing the previously defined architecture.

\color{black}

\subsection{Error analysis of encoder-decoder networks}\label{se:error_analysis}
This subsection is dedicated to the error estimate of the general encoder-decoder networks presented in Section \ref{sec:general_encoder-decoder}, i.e., $\Psi_{\theta} - \Psi^{\dagger}$.  	
To this end, we first make the following assumptions on spaces, encoders/decoders, and the objective mapping $\Psi^\dagger$.

\begin{assumption}\label{ass1}
	Let $\mathcal{V}$, $\mathcal{U}$, $\mathcal{V}_h$, and $\mathcal{U}_h$ be Banach spaces. We assume that
	\begin{enumerate}
		\item (\textbf{Linearity}) The encoder \( E_{\mathcal{V}} \colon \mathcal{V} \to \mathcal{V}_h \) and the decoder \( D_{\mathcal{U}} \colon \mathcal{U}_h \to \mathcal{U} \) are bounded linear operators that are invertible in the generalized sense (see Definition~1.38 in \cite{boichuk2016generalized}).
		
		\item (\textbf{Generalized inversion}) The decoder \( D_{\mathcal{V}} \colon \mathcal{V}_h \to \mathcal{V} \) and the encoder \( E_{\mathcal{U}} \colon \mathcal{U} \to \mathcal{U}_h \) are generalized inverses of \( E_{\mathcal{V}} \) and \( D_{\mathcal{U}} \), respectively, in the sense that (see Definition~1.38 and Equation~(1.7) in \cite{boichuk2016generalized})
		\[
		D_{\mathcal{V}} \circ E_{\mathcal{V}} \circ D_{\mathcal{V}} = D_{\mathcal{V}}, 
		\qquad
		D_{\mathcal{U}} \circ E_{\mathcal{U}} \circ D_{\mathcal{U}} = D_{\mathcal{U}}.
		\]
		
		\item (\textbf{Continuity}) The operator \( \Psi^{\dagger} \colon \mathcal{V} \to \mathcal{U} \) is \(\beta\)-H{\"o}lder continuous with H{\"o}lder constant \( L_{\Psi^{\dagger}} \) for some \( \beta \in (0,1] \).
		
		\item (\textbf{Universal approximation property})  
		Define the operator \( \psi \colon \mathcal{V}_h \to \mathcal{U}_h \) by
		\[
		\psi(\zeta) = E_{\mathcal{U}} \circ \Psi^{\dagger} \circ D_{\mathcal{V}}(\zeta).
		\]
		For any compact subset \( \mathcal{K} \subset \mathcal{V}_h \) and any \( \epsilon > 0 \), there exists a neural network \( \psi_{\theta} \colon \mathcal{V}_h \to \mathcal{U}_h \), with appropriate architecture and parameters \( \theta \), such that
		\[
		\|\psi_{\theta}(v) - \psi(v)\|_{\mathcal{U}_h} \leq \epsilon,
		\qquad \forall v \in \mathcal{K}.
		\]
	\end{enumerate}
\end{assumption}
Then, we define the following two crucial \textit{consistency errors} of the encoding$-$decoding schemes:
\begin{align}
	&d_{1} (v) \coloneqq   \| D_{\mathcal{V}}\circ E_{\mathcal{V}} (v) -v \|_{\mathcal{V}}, \qquad\forall v\in \mathcal{V}  \label{eq:distances_1},\\
	&d_{2} (u) \coloneqq \|D_{\mathcal{U}} \circ E_{\mathcal{U}} (u) -u  \|_{\mathcal{U}},  \qquad\!\forall u\in \mathcal{U}\label{eq:distances_2}.
\end{align}

Before presenting the main error estimate results, we provide an illustrative example 
(see the remark below) that meets all of the above assumptions. Moreover, under 
additional regularity conditions on $v$ and $u$, its consistency errors vanish 
as the latent space is refined.
\begin{remark}\label{rem:example}
	Consider the \emph{stationary} reaction–diffusion equation posed on the torus \( \mathbb{T}^d \):
	\[
	-\Delta u + c(x)\,u = f(x), \quad  x \in \mathbb{T}^d,
	\]
	where \( c \in \mathcal{C}^{0,\alpha}(\mathbb{T}^d) \) is uniformly positive and \( f \in \mathcal{C}^{0,\alpha}(\mathbb{T}^d) \) for some \( \alpha \in (0,1] \) \footnote{For any \(k \in \mathbb{Z}_+\) and \(\alpha \in (0,1]\), the H{\"o}lder space \(\mathcal{C}^{k,\alpha}(\mathbb{T}^d)\) is defined as
		\[
		\mathcal{C}^{k,\alpha}(\mathbb{T}^d)
		\coloneqq
		\bigl\{\, u \in \mathcal{C}^k(\mathbb{T}^d) \;\colon\; \|u\|_{\mathcal{C}^{k,\alpha}} < \infty \,\bigr\},
		\]
		where the H{\"o}lder norm is given by
		\[
		\|u\|_{\mathcal{C}^{k,\alpha}}
		\coloneqq
		\sum_{\|I\|_{\ell^1} \le k} \|\partial^I u\|_{L^\infty(\mathbb{T}^d)}
		\;+\;
		\sum_{\|I\|_{\ell^1} = k}
		\sup_{\substack{x,y \in \mathbb{T}^d \\ x \neq y}}
		\frac{\bigl| \partial^I u(x) - \partial^I u(y) \bigr|}{\|x - y\|^\alpha}.
		\]
		Here, the notation \(\partial^I u\) represents the partial derivative of \(u\) associated with the multi-index \(I \in \mathbb{Z}_+^d\).}.
	By the classical Schauder estimate \cite[Thm.~4.3.2]{krylov1996lectures}, 
	together with the Sobolev estimate for elliptic equations 
	\cite[Thm.~1(ii), Ch.~5.1]{krylov2024lectures} and Morrey's inequality \cite[Thm.~4.12 Part II]{adams2003sobolev}, 
	the unique solution $u$ satisfies
	\[
	u \in \mathcal{C}^{2,\alpha}(\mathbb{T}^d),
	\qquad
	\|u\|_{\mathcal{C}^{2,\alpha}} \leq C_1 \|f\|_{\mathcal{C}^{0,\alpha}},\qquad
	\|u\|_{\mathcal{C}^1(\mathbb{T}^d)} \leq C_2 \|f\|_{\mathcal{C}(\mathbb{T}^d)},
	\]
	where $C_1$ and $C_2$ are constants independent of $f$.
	Therefore, the solution map \( \Psi^\dagger\colon f \mapsto u \) is Lipschitz (and hence H{\"o}lder) continuous from \( \mathcal{C}^{0,\alpha}(\mathbb{T}^d) \) to \( \mathcal{C}^{2,\alpha}(\mathbb{T}^d) \), and from \( \mathcal{C}(\mathbb{T}^d) \) to \( \mathcal{C}^1(\mathbb{T}^d) \).
	Then we consider the following setups for the encoder-decoder network:
	\begin{itemize}
		\item Let \( \mathcal{U} = \mathcal{V} = \mathcal{C}(\mathbb{T}^d) \), and set \( \mathcal{U}_h = \mathcal{V}_h = \mathbb{R}^{N^d} \) for mesh size \( h = 1/N \ll 1 \);
		\item Define the encoder \( E_{\mathcal{V}} \) and decoder \( D_{\mathcal{U}} \) via finite-difference stencils and interpolation by $Q_1$ finite-element basis on a uniform grid with step size $h$;
		\item Take \( D_{\mathcal{V}} = D_{\mathcal{U}} \) and \( E_{\mathcal{U}} = E_{\mathcal{V}} \), and we can verify that the generalized inversion condition is satisfied (See Lemma~\ref{lemma:inverse-stationary} and Remark~\ref{rem:fd-fem} for a rigorous proof).
	\end{itemize}
	Under this construction, it follows directly from the Lipschitz continuity of \( \Psi^{\dagger} \) and continuity of the encoder/decoder pair that the operator
	\[
	\psi(v) = E_{\mathcal{U}} \circ \Psi^{\dagger} \circ D_{\mathcal{V}}(v)
	\]
	is continuous. By the classical universal approximation theorem (e.g., \cite{cybenko1989approximation}), this continuity ensures that \( \psi \) can be uniformly approximated by shallow neural networks \footnote{A shallow neural network mapping \(\mathbb{R}^n\) to \(\mathbb{R}^m\) is given by
		\[
		f_{\text{shallow}}(x;\,\theta) \coloneqq W_2 \, \sigma(W_1 x + b),
		\]
		where \(\sigma\) denotes the activation function applied componentwise, and the parameters $\theta=(W_1,W_2,b)$ with \(W_1 \in \mathbb{R}^{P \times n}\), \(W_2 \in \mathbb{R}^{m \times P}\), and \(b \in \mathbb{R}^P\). Here, \(P\in \mathbb{N}\) represents the number of neurons in the hidden layer. The universal approximation property of this neural network is understood in the sense that the number of neurons \(P\) tends to infinity.
	} on any compact subset of \( \mathcal{V}_h \), thereby validating Assumption~\ref{ass1}(4).
	Finally, the consistency errors associated with the encoder and decoder satisfy the following estimates: 
	\begin{equation*}
		\sup_{\begin{subarray}{c}
				v \in \mathcal{C}^{0,\alpha}(\mathbb{T}^d) \\
				\|v\|_{\mathcal{C}^{0,\alpha}} \le 1
		\end{subarray}}
		d_1(v) \; \leq \;  C h^{\alpha}, \qquad   \sup_{\begin{subarray}{c}
				u \in \mathcal{C}^{1}(\mathbb{T}^d) \\
				\|u\|_{\mathcal{C}^{1}} \le 1
		\end{subarray}}
		d_2(u) \;\le\; C h,
	\end{equation*}
	where $C>0$ is a constant independent of $h$. 
	In particular, the consistency errors vanish at the rate $h^{\alpha}$ and can thus be made arbitrarily small as $h \to 0$.
	
\end{remark}

The main result on the error analysis of the generic encoder-decoder network is presented in the following theorem.

\begin{theorem}\label{thm:main}
	Let Assumption~\ref{ass1} hold. Then, for any compact subset \( \mathcal{K} \subset \mathcal{V}_h \) and any \( \epsilon > 0 \), there exists a neural network function \( \psi_{\theta} \colon \mathcal{V}_h \to \mathcal{U}_h \), with a suitable architecture and parameters \( \theta \), such that for any \( v \in \mathcal{V} \) satisfying \( E_{\mathcal{V}}(v) \in \mathcal{K} \), we have
	\begin{equation}\label{eq:main}
	\begin{aligned}
		\| \Psi_{\theta}(v) - \Psi^{\dagger}(v) \|_{\mathcal{U}} \,\leq\,
		\underbrace{L_{\Psi^{\dagger}} \, d_1(v)^{\beta}}_{\textnormal{encoding-decoding error in } \mathcal{V}} 
		\;&+\;
		\underbrace{d_2\big(\Psi^{\dagger} \circ D_{\mathcal{V}} \circ E_{\mathcal{V}}(v)\big)}_{\textnormal{encoding-decoding error in } \mathcal{U}} 
		\;\\
		&+\;
		\underbrace{\|D_{\mathcal{U}}\| \, \epsilon}_{\textnormal{NN approximation error}},
	\end{aligned}
	\end{equation}
	where \( \Psi_{\theta} \) is defined in~\eqref{eq:Psi_theta}, and the functions \( d_1 \) and \( d_2 \) are given in~\eqref{eq:distances_1}–\eqref{eq:distances_2}.
\end{theorem}

\begin{proof}
	Fix any compact set $\mathcal{K}\subset \mathcal{V}_h$ and $\epsilon >0$.
	Let \(v\in \mathcal{V}\) such that $E_{\mathcal{V}}(v)\in \mathcal{K}$. We first decompose the difference
	\[
	\Psi_{\theta}(v) - \Psi^{\dagger}(v) = D_{\mathcal{U}}\circ\psi_{\theta}\circ E_{\mathcal{V}}(v) - \Psi^{\dagger}(v)
	\]
	as follows:
	\begin{equation}\label{eq:proof_1}
		\Psi_{\theta}(v) - \Psi^{\dagger}(v) := \gamma_1(v) + \gamma_2(v),
	\end{equation}
	where
	\begin{align*}
		\gamma_1(v) := D_{\mathcal{U}}\circ\psi_{\theta}\circ E_{\mathcal{V}}(v) - D_{\mathcal{U}}\circ\psi\circ E_{\mathcal{V}}(v), \quad
		\gamma_2(v) := D_{\mathcal{U}}\circ\psi\circ E_{\mathcal{V}}(v) - \Psi^{\dagger}(v).
	\end{align*}
	By Assumption~\ref{ass1}(4), the neural network approximation error satisfies
	\begin{equation}\label{eq:proof_2}
		\|\gamma_1(v)\| \le \|D_{\mathcal{U}}\| \, \epsilon.
	\end{equation}
	
	Next, define
	\[
	v_1 \coloneqq D_{\mathcal{V}}\circ E_{\mathcal{V}}(v) \quad \text{and} \quad u_1 \coloneqq D_{\mathcal{U}}\circ E_{\mathcal{U}}\circ \Psi^{\dagger}(v_1).
	\]
	Since \(v_1\) and \(u_1\) lie in the ranges of \(D_{\mathcal{V}}\) and \(D_{\mathcal{U}}\), respectively, Assumption~\ref{ass1}(2) implies that
	\begin{equation*}
		D_{\mathcal{V}}\circ E_{\mathcal{V}}(v_1) = v_1 \quad \text{and} \quad D_{\mathcal{U}}\circ E_{\mathcal{U}}(u_1) = u_1.
	\end{equation*}
	We now further decompose \(\gamma_2(v)\) as
	\begin{equation}\label{eq:proof_3}
		\gamma_2(v) = D_{\mathcal{U}}\circ\psi\circ E_{\mathcal{V}}(v) - \Psi^{\dagger}(v)
		= \gamma_{2,1}(v) + \gamma_{2,2}(v),
	\end{equation}
	where
	\begin{align*}
		\gamma_{2,1}(v) &:= D_{\mathcal{U}}\circ\psi\circ E_{\mathcal{V}}(v)
		- D_{\mathcal{U}}\circ\psi\circ E_{\mathcal{V}}(v_1)
		- \Bigl( \Psi^{\dagger}(v) - \Psi^{\dagger}(v_1) \Bigr), \\
		\gamma_{2,2}(v) &:= D_{\mathcal{U}}\circ\psi\circ E_{\mathcal{V}}(v_1)
		- u_1 - \Bigl( \Psi^{\dagger}(v_1) - u_1 \Bigr).
	\end{align*}
	Using the definition of \(\psi\) and the fact that \(D_{\mathcal{V}}\circ E_{\mathcal{V}}(v_1) = v_1\), we obtain
	\begin{equation}\label{eq:proof_4}
		\|\gamma_{2,1}(v)\|_{\mathcal{U}} = \Bigl\| \Psi^{\dagger}(v) - \Psi^{\dagger}(v_1) \Bigr\|_{\mathcal{U}} \le L_{\Psi^{\dagger}}\, d_1(v)^{\beta},
	\end{equation}
	where the inequality follows from Assumption~\ref{ass1}(3). Similarly, by the definitions of \(\psi\) and \(u_1\), we have
	\begin{equation*}
		\|\gamma_{2,2}(v)\|_{\mathcal{U}} = \Bigl\| \Psi^{\dagger}(v_1) - D_{\mathcal{U}}\circ E_{\mathcal{U}}\bigl( \Psi^{\dagger}(v_1) \bigr) \Bigr\|_{\mathcal{U}}.
	\end{equation*}
	Recalling the definition of \(d_2\) from \eqref{eq:distances_2}, we deduce
	\begin{equation}\label{eq:proof_5}
		\|\gamma_{2,2}(v)\|_{\mathcal{U}} = d_2\Bigl(\Psi^{\dagger}(v_1)\Bigr)
		= d_2\Bigl( \Psi^{\dagger}\circ D_{\mathcal{V}}\circ E_{\mathcal{V}}(v) \Bigr).
	\end{equation}
	Then the desired result \eqref{eq:main} follows from \eqref{eq:proof_1}-\eqref{eq:proof_5} directly.
\end{proof}

\begin{corollary}\label{cor:example}
	Under the setting of Remark~\ref{rem:example}, fix the regularity order $\alpha\in (0,1)$ and fix a discretization step size \(h = 1/N\) with \(N\in\mathbb{N}_+\).  Then there exists a shallow neural network 
	\( \psi_{\theta}\colon \mathbb{R}^{N^d} \;\to\;\mathbb{R}^{N^d},\)
	with a sufficiently large number of neurons and appropriately chosen parameters \(\theta\), such that for every 
	\[
	f \in C^{0,\alpha}(\mathbb{T}^d),
	\qquad
	\|f\|_{C^{0,\alpha}(\mathbb{T}^d)}\le1,
	\]
	the following error bound holds:
	\[
	\bigl\|\Psi_{\theta}(f) - \Psi^{\dagger}(f)\bigr\|_{\mathcal{C}(\mathbb{T}^d)}
	\;\le\;
	C\,h^{\alpha},
	\]
	where the constant \(C>0\) is independent of both \(h\) and \(f\).
\end{corollary}
\begin{proof}
	Throughout the proof, the notation $\lesssim$ indicates an inequality $\leq$ 
	up to a multiplicative constant on the right-hand side that is independent 
	of both $h$ and $f$.
	Recall the estimate \eqref{eq:main} from Theorem~\ref{thm:main}.

	For the first term $L_{\Psi^{\dagger}} \, d_1(v)^{\beta}$, we invoke the Lipschitz 
	continuity of $\Psi^\dagger$ (with $\beta=1$), together with the consistency-error 
	bound for $d_1$ from Remark~\ref{rem:example}, to obtain
	\begin{equation*}
		L_{\Psi^{\dagger}} \, d_1(f)^{\beta} \;\lesssim\; h^{\alpha}.
	\end{equation*}
	
	For the second term, $d_2\big(\Psi^{\dagger} \circ D_{\mathcal{V}} \circ E_{\mathcal{V}}(f)\big)$, 
	recalling the definition of encoder-decoder operators in Remark \ref{rem:example}, we deduce
	\[
	\|D_{\mathcal{V}} \circ E_{\mathcal{V}}(f)\|_{\mathcal{C}(\mathbb{T}^d)} 
	\;\lesssim\; \|f\|_{\mathcal{C}(\mathbb{T}^d)} 
	\;\lesssim\; \|f\|_{\mathcal{C}^{0,\alpha}(\mathbb{T}^d)}.
	\]
	By elliptic regularity in Sobolev spaces \cite[Thm.~1 (ii), Chp.~5.1]{krylov2024lectures} together with Morrey’s inequality \cite[Thm.~4.12 Part II]{adams2003sobolev}, it follows that
	\begin{equation*}
		\|\Psi^{\dagger} \circ D_{\mathcal{V}} \circ E_{\mathcal{V}}(f)\|_{\mathcal{C}^1(\mathbb{T}^d)}\;\lesssim \|D_{\mathcal{V}} \circ E_{\mathcal{V}}(f)\|_{\mathcal{C}(\mathbb{T}^d)}\; 
		\;\lesssim\; \|f\|_{\mathcal{C}^{0,\alpha}(\mathbb{T}^d)}.
	\end{equation*}
	Hence, the consistency error for $d_2$ in Remark \ref{rem:example} leads
	\begin{equation*}
		d_2\!\left(\Psi^{\dagger} \circ D_{\mathcal{V}} \circ E_{\mathcal{V}}(f)\right)  
		\;\lesssim\; h.
	\end{equation*}
	
	For the third term, $\|D_{\mathcal{U}}\|\,\epsilon$, note first that $D_{\mathcal{U}}$ is a bounded 
	operator with norm $1$. Moreover, $E_{\mathcal{V}}(f)$ takes values in a fixed compact set in 
	$\mathbb{R}^{N^d}$, since $f$ is bounded and $E_{\mathcal{V}}$ is continuous. The mapping 
	$E_{\mathcal{U}} \circ \Psi^{\dagger}\circ D_{\mathcal{V}}$ restricted to this compact set can be 
	approximated arbitrarily well by a shallow neural network, thanks to the universal approximation 
	theorem~\cite{cybenko1989approximation}. Taking $\epsilon = h^{\alpha}$, we conclude that
	\[
	\|D_{\mathcal{U}}\|\,\epsilon \;\lesssim \; h^{\alpha}.
	\]
	Combining the upper bounds on all three terms yields the desired estimate.
\end{proof}

For stationary PDEs, Remark \ref{rem:example} and Corollary \ref{cor:example} provide a canonical framework for constructing and analyzing encoder-decoder networks. Nevertheless, extending this analysis framework to non-stationary PDEs is more delicate. A major challenge is temporal discretization: fully discrete schemes often require careful stability control. As an alternative, one may adopt a time-continuous (semi-discrete) formulation, but this forces the latent spaces to be infinite-dimensional functional spaces, demanding more advanced universal approximation results in such settings.

Recently,
approximation guarantees for dynamical systems via NODEs has been studied in \cite{li2024universal}, providing complementary insights into ODE-driven mappings between function spaces. Leveraging this advance, we introduce the NODE-ONet framework for PDEs in Section \ref{se: architecture}. Although the error analysis for an NODE-ONet depends intricately on the specific PDE under consideration and is technically involved, the overall framework adheres to Theorem \ref{thm:main}. A rigorous investigation of approximation errors is beyond the scope of this work and will be pursued in future studies. Consequently, we focus here on the algorithmic design and the numerical investigation of NODE-ONets for various PDEs. 

\section{The NODE-ONet Framework}\label{se: architecture}
This section introduces the NODE-ONet framework, within the general architecture of the previously discussed encoder-decoder networks, to approximate the mapping from the parameters of \eqref{e_pde} to its solution $u$.  

\subsection{Stationary case}\label{se:static}
Before addressing the general case \eqref{e_pde}, we first consider its stationary counterpart:
\begin{equation}\label{e_pde_stationary}
	\left\{
	\begin{aligned}
		&\mathcal{L}[a](u)(x) = f(x), &\quad \forall x \in \Omega,\\
		&\mathcal{B}u(x) = u_b(x), &\quad \forall x \in \partial\Omega,
	\end{aligned}
	\right.
\end{equation}
to provide some insights into the design of the NODE-ONets. 
Here, $a$ and $f$ depend only on space. To simplify and unify the analysis, we assume $\Omega\subset \mathbb{R}^d$ is compact and $u_b$ is fixed. Thus, the relevant parameters in this setting are $a$ and $f$. Let $\mathcal{C}(\Omega)$ be the space of continuous functions on $\Omega$. We assume that the parameters $a, f \in \mathcal{E}\subset\mathcal{C}(\Omega)$ and that the stationary equation \eqref{e_pde_stationary} admits a unique classical solution for all $a, f\in\mathcal{E}$. Consequently, we assume that  $\mathcal{V}$ and $\mathcal{U}$ are subspaces of $\mathcal{C}(\Omega)$.  

As outlined in Section~\ref{sec:general_encoder-decoder}, the key components of an encoder-decoder network involve determining suitable latent spaces $\mathcal{V}_h$ and $\mathcal{U}_h$, designing the neural network approximation $\psi_\theta$, and defining the encoder $E_\mathcal{V}$ and the decoder $D_{\mathcal{U}}$. 

First, the latent spaces are chosen as finite-dimensional Euclidean spaces with dimensions $d_{\mathcal{V}}$ and $d_{\mathcal{U}}$, respectively:
\begin{equation*}
	\mathcal{V}_h = \mathbb{R}^{d_{\mathcal{V}}}, \quad \mathcal{U}_h = \mathbb{R}^{d_{\mathcal{U}}}.
\end{equation*}
Let $\psi_\theta$ be a neural network such that Assumption \ref{ass1}(4) holds. Then, the encoder and the decoder for the stationary case are defined as follows:
\begin{enumerate}
	\item[(S1).] \textbf{Stationary Encoding (Space Discretization):}  
	Given any function $v\in\mathcal{C}(\Omega)$, the stationary encoder is defined as
	\[
	E^s_{\mathcal{V}}: \mathcal{C}(\Omega) \to \mathbb{R}^{d_{\mathcal{V}}}, \quad v \mapsto \bigl(L_{\ell}(v)\bigr)_{\ell=1}^{d_{\mathcal{V}}},
	\]
	where $L_{\ell}$ are bounded linear operators on $\mathcal{C}(\Omega)$. Therefore, by the Riesz representation theorem, each $L_{\ell}$ has an integral representation written as
	\[
	L_{\ell}(v) = \int_{\Omega} v \, d\mu_{\ell}, \quad \text{for some } \mu_{\ell} \in \mathcal{M}(\Omega),
	\]
	where $\mathcal{M}(\Omega)$ denotes the space of Radon measures supported in $\Omega$. Notably, if we take $\mu_{\ell} = \delta_{x_{\ell}}$, the operator reduces to evaluating $v$ at the point $x_{\ell}$.
	
	\item[(S2).] \textbf{Stationary Decoding (Reconstruction):}  
	Given $(\psi_j)_{j=1}^{d_{\mathcal{U}}}\in\mathbb{R}^{d_{\mathcal{U}}}$ the output of the neural network $\psi_\theta$, the stationary decoder $D^s_{\mathcal{U}}$ is defined as the following interpolation form:
	\[
	D^s_{\mathcal{U}}: \mathbb{R}^{d_{\mathcal{U}}} \to \mathcal{C}(\Omega), \quad (\psi_j)_{j=1}^{d_{\mathcal{U}}} \mapsto \sum_{j=1}^{d_{\mathcal{U}}} \alpha_j(x; \theta_{\alpha})\, \psi_j, \quad \text{for } x \in \Omega.
	\]
	Here, the set 
	\begin{equation}\label{eq:basis_decoder}
		\bm{\alpha}(x) := \{\alpha_j(x; \theta_{\alpha})\}_{j=1}^{d_{\mathcal{U}}} \in \mathbb{R}^{d_{\mathcal{U}}}
	\end{equation}
	represents either the values of a set of predefined spatial basis functions at $x$ (e.g., finite element or Fourier bases) or the output of a neural network $\mathcal{N}_{\theta_{\alpha}}: \Omega \to \mathbb{R}^{d_{\mathcal{U}}}$ parameterized by $\theta_{\alpha} \in \mathbb{R}^{p_\alpha}$.
\end{enumerate}

\begin{lemma}\label{lemma:inverse-stationary}
	Let $\mathcal{V}$ and $\mathcal{U}$ be non-empty closed subspaces of $\mathcal{C}(\Omega)$. The encoder $E^s_{\mathcal{V}}$ and the decoder $D^s_{\mathcal{U}}$, defined above, satisfy Assumption~\ref{ass1}(1). Moreover, their generalized inverses, $D^s_{\mathcal{V}}$ and $E^s_{\mathcal{U}}$, which satisfy Assumption~\ref{ass1}(2), are given by:
	\begin{equation*}
		D^s_{\mathcal{V}}(z) = \sum_{k=1}^{d_{\mathcal{V}}} z_{k} f_{k} , \quad \forall \,z\in \mathbb{R}^{d_{\mathcal{V}}}, \quad \textnormal{where } \int_{\Omega} f_{k} \,d\,\mu_{\ell} = \begin{cases}
			1, \quad \textnormal{if } k= \ell;\\
			0, \quad \textnormal{otherwise}.
		\end{cases}
	\end{equation*}
	and 
	\begin{equation*}
		E^s_{\mathcal{U}}(u) = \left( \int_{\Omega} u \,d\,m_i \right)_{i=1}^{d_{\mathcal{U}}} ,\; \forall \,u\in \mathcal{U}, \quad \textnormal{where } \int_{\Omega} \alpha_j \,d\,m_i = \begin{cases}
			1, \quad \textnormal{if }i =j;\\
			0, \quad \textnormal{otherwise}.
		\end{cases}
	\end{equation*}
\end{lemma}
\begin{proof}
	The proof follows from a straightforward calculation.
\end{proof}

The above lemma establishes the pseudo-inverse property for the standard finite-difference mesh and the $Q_1$ finite-element basis, as illustrated in the following remark. This, in turn, justifies the validity of the stationary elliptic equation example given in Remark~\ref{rem:example}.

\begin{remark}[Uniform finite‐difference mesh and $Q_1$ finite‐element basis]\label{rem:fd-fem}
	Let \(\Omega\) be a hypercube in $\mathbb{R}^d$, and let
	\(
	\Omega_h = \{x_i\}_{i=1,\dots,N^d}
	\)
	be the corresponding uniform finite‐difference grid with mesh size \(h\).  
	We take $d_{\mathcal{V}} = d_{\mathcal{U}} = N^d$.
	Then the grid trace operator and the $Q_1$ finite‐element interpolant form a natural encoder–decoder pair satisfying Lemma~\ref{lemma:inverse-stationary}.  Specifically, in the notation of Lemma~\ref{lemma:inverse-stationary}, we define
	\[
	\mu_i = m_i = \delta_{x_i},
	\qquad
	\alpha_i(x) =f_i(x)= \prod_{j=1}^d \phi_{i,j}(x),\qquad \forall \, i=1,\ldots, N^d, \quad \forall\, x\in \Omega,
	\]
	where \(\phi_{i,j}\) is the one‐dimensional $P_1$ hat function centered at \(x_i\) in the coordinate direction \(e_j\) with support of length \(2h\). Following from Lemma \ref{lemma:inverse-stationary}, by the core identity:
	\[
	\int_{\Omega} \alpha_i  \, d\,\mu_j = \alpha_i(x_j) = \begin{cases}
		1, \quad \textnormal{if } i= j;\\
		0, \quad \textnormal{otherwise}.
	\end{cases}
	\]
	Assumption~\ref{ass1}(2) holds for the resulting pairs $(E_{\mathcal{V}}^s, D_{\mathcal{V}}^s)$ and $(E_{\mathcal{U}}^s, D_{\mathcal{U}}^s)$.
\end{remark}

\subsection{Non-stationary case}\label{se_NODE-ONet} 
Let us now present the proposed NODE-ONet framework for the non-stationary case \eqref{e_pde}, adhering to the general encoder-decoder architecture. In particular, we shall elaborate on the design of suitable latent spaces $\mathcal{V}_h$ and $\mathcal{U}_h$, an NODE-induced neural network $\psi_\theta$, the encoder $E_\mathcal{V}$, and the decoder $D_{\mathcal{U}}$.

For convenience, we assume that $\Omega$ is compact, $u_b$ is fixed, and the parameters $a, f\in\mathcal{K}\subset\mathcal{C}([0,T];\,\mathcal{C}(\Omega))$, $u_0\in \mathcal{C}(\Omega)$. First, inspired by the latent spaces in the stationary case, we define
\begin{equation*}
	\mathcal{V}_h =  \mathcal{C} ([0,T];\, \mathbb{R}^{d_{\mathcal{V}}}) \quad \text{and} \quad  \mathcal{U}_h = \mathcal{C} ([0,T];\, \mathbb{R}^{d_{\mathcal{U}}}),
\end{equation*}
where $d_{\mathcal{V}}$ and $d_{\mathcal{U}}$ are positive constants.  We then propose the NODE-ONet framework, consisting of the following three components:
\begin{enumerate}
	\item[(NS1).] \textbf{Encoding (Space Discretization):} Given an input function $v\in\mathcal{C}([0,T];\, \mathcal{C}(\Omega))$, by taking the stationary encoder $E^s_{\mathcal{V}}$ for each $v(t)\in\mathcal{C}(\Omega), t\in [0,T]$, we obtain the encoder for the non-stationary case: \[E_{\mathcal{V}}: \mathcal{C}([0,T];\, \mathcal{C}(\Omega)) \to 
	\mathcal{C} ([0,T];\, \mathbb{R}^{d_{\mathcal{V}}} ),\, v\mapsto \bm{v} = (v_{\ell})_{\ell=1}^{d_{\mathcal{V}}}, \]
	where
	\[
	v_{\ell} (t) = (E^s_{\mathcal{V}} (v(t,\cdot)))_{\ell} = L_{\ell}(v(t,\cdot)), \quad \text{for } t\in [0,T] \text{ and } \ell=1,\ldots, d_{\mathcal{V}}. 
	\]
	\item [(NS2).] \textbf{NODE (Approximation):} Given
	$\bm{v}\in\mathcal{C}([0,T];\, \mathbb{R}^{d_{\mathcal{V}}})$,
	let $\bm{\psi}\in \mathcal{C}([0,T];\, \mathbb{R}^{d_{\mathcal{U}}})$ satisfy the following NODE:
	\begin{equation}
		\label{eq:NODE}
		\begin{dcases}
			\dot{\bm{\psi}}(t) = \mathcal{N}_{\theta_{\psi}}\bigl(\bm{\psi}(t),\, \mathcal{P}_v\bm{v}(t), \,t\bigr), \quad \text{for } t\in [0,T], \\
			\bm{\psi}(0) = \mathcal{P}_u \, E_{\mathcal{U}}^s(u_0) \in \mathbb{R}^{d_{\mathcal{U}}}.
		\end{dcases}
	\end{equation}
	Here, $\mathcal{N}_{\theta_{\psi}} : \mathbb{R}^{d_{\mathcal{U}}} \times \mathbb{R}^{d_{\mathcal{U}}} \times \mathbb{R}_{+}\to \mathbb{R}^{d_{\mathcal{U}}}$ denotes a neural network parameterized by $\theta_{\psi} \subset \mathbb{R}^{p_\psi}$, while $\mathcal{P}_v,\, \mathcal{P}_u \in \mathbb{R}^{d_{\mathcal{U}} \times d_{\mathcal{V}}}$ are trainable transformation matrices (see Section \ref{se:pe_node} for further demonstrations).
	We then introduce the NODE operator, parameterized by $\theta_\Psi$, $\mathcal{P}_u$, and $\mathcal{P}_v$, as
	\[
	\text{NODE}\,({\theta_\Psi,\mathcal{P}_u,\mathcal{P}_v}) : \mathcal{C}([0,T];\, \mathbb{R}^{d_{\mathcal{V}}}) \to \mathcal{C}([0,T];\, \mathbb{R}^{d_{\mathcal{U}}}), \quad \bm{v}\mapsto \bm{\psi},
	\]
	and employ the NODE-induced neural network $\psi_\theta:=\text{NODE}\,({\theta_\Psi,\mathcal{P}_u,\mathcal{P}_v})$ to approximate the mapping 
	\(
	\psi = E_\mathcal{U} \circ \Psi^\dagger \circ D_\mathcal{V}.
	\)

	\item[(NS3).] \textbf{Decoding (Reconstruction):} Given
	$\bm{\psi}: =(\psi_j)_{j=1}^{d_{\mathcal{U}}}\in\mathcal{C}([0,T];\, \mathbb{R}^{d_{\mathcal{U}}})$ the output of an NODE-induced neural network $\psi_\theta$, the decoder $D_{\mathcal{U}}$ is defined by
	\[
	D_{\mathcal{U}} \colon\, \mathcal{C}([0,T];\, \mathbb{R}^{d_{\mathcal{U}}}) \to  \mathcal{C}([0,T];\, \mathcal{C}(\Omega)),\, \bm{\psi} \mapsto \Psi,
	\]
	where 
	\[
	D^s_{\mathcal{U}}(\bm{\psi}(t))(x) = \sum_{j=1}^{d_{\mathcal{U}}} \alpha_j(x; \theta_{\alpha})  \,\psi_j(t) \quad \text{for } (t,x)\in [0,T]\times \Omega,
	\]
	with $\{\alpha_j\}_{j=1}^{d_{\mathcal{U}}}$ a set of basis functions as defined in \eqref{eq:basis_decoder}.   
\end{enumerate}

With the above constructions, the proposed NODE-ONet framework reads: for any $  v\in {\mathcal{V}}\subseteq \mathcal{C}([0,T];\, \mathcal{C}(\Omega))$ and any $(t,x) \in [0,T]\times \Omega$,
\begin{align*}
	\Psi_{\text{NODE-ONet}}(v;\theta)(t, x) := \sum_{j=1}^{d_{\mathcal{U}}}\alpha_j(x; \theta_{\alpha}) \psi_j(t,\bm{v};\theta_{\psi},\mathcal{P}_u,\mathcal{P}_{v}),
\end{align*}
where $\bm{v} = E_{\mathcal{V}}(v)$ and $\theta=\{\theta_\psi,\theta_\alpha,\mathcal{P}_v,\mathcal{P}_{u}\}$ collects all the trainable parameters, see Figure \ref{fig: deepNODE}.

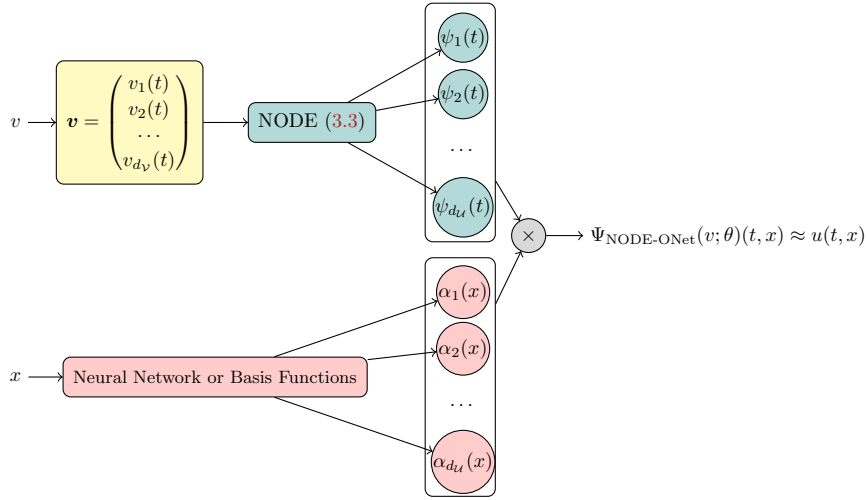
\begin{figure}[htpb]
	\label{fig:DnodeON}
	\centering
	\begin{tikzpicture}[global scale =0.75]
		\centering
		\node at(10,9)    (24) {$v$};
		\node at(12,9)      (25) [zbox,rounded corners=1mm,fill=yellow!30]{$\bm{v}=\begin{pmatrix}
				v_1(t)\\v_2(t)\\ \cdots\\v_{d_{\mathcal{V}}}(t)
			\end{pmatrix}$};
		\node at(15.2,9)      (26)[zbox,rounded corners=1mm,fill=teal!30]{NODE \eqref{eq:NODE} };
		\node at(13.5,4.5)    (27)[zbox,rounded corners=1mm,fill=red!20]{\small Neural Network or Basis Functions};
		\node at(17.8,9)    (28)[rectangle,rounded corners=1mm, minimum width =35pt, minimum height =120pt, inner sep=5pt,draw=black]  {} ;        
		\node at(17.8,4.5)  (29)[rectangle,rounded corners=1mm, minimum width =35pt, minimum height =120pt, inner sep=5pt,draw=black]  {} ; 
		\node at(17.85,10.5) (30)[circle, minimum width =17pt, minimum height =17pt, inner sep=0.1pt, fill=gray!30,draw=black,fill=teal!30]{$\psi_1(t)$};
		\node at(17.85,9.5)  (31)[circle, minimum width =17pt, minimum height =17pt, inner sep=0.1pt, fill=gray!30,draw=black,fill=teal!30]{$\psi_2(t)$};
		\node at(17.85,8.5)          {$\cdots$};
		\node at(17.85,7.5)  (32)[circle, minimum width =17pt, minimum height =17pt, inner sep=0.1pt, fill=gray!30,draw=black,fill=teal!30]{$\psi_{d_{\mathcal{U}}}(t)$};
		\node at(17.85,6)    (33)[circle, minimum width =17pt, minimum height =17pt, inner sep=0.1pt, fill=gray!30,draw=black,fill=red!20]{$\alpha_1(x)$};
		\node at(17.85,5)   (34)[circle, minimum width =17pt, minimum height =17pt, inner sep=0.1pt, fill=gray!30,draw=black,fill=red!20]{$\alpha_2(x)$};
		\node at(17.85,4)          {$\cdots$};
		\node at(17.85,3)   (35)[circle, minimum width =17pt, minimum height =17pt, inner sep=0.1pt, fill=gray!30,draw=black,,fill=red!20]{$\alpha_{d_{\mathcal{U}}}(x)$};
		\node at(19,7)    (36)[circle, minimum width =17pt, minimum height =17pt, inner sep=0.1pt, fill=gray!30,draw=black]{$\times$};
		\node at(22.5,7)      (37){$\Psi_{\text{NODE-ONet}}(v;\theta)(t, x) \approx u(t,x)$};
		\node at(10,4.5)  (38){$x$};
		
		\draw[->] (24) --(25);
		\draw[->] (25) --(26);
		\draw[->] (26) --(30);
		\draw[->] (26) --(31);
		\draw[->] (26) --(32);
		\draw[->] (27) --(33);
		\draw[->] (27) --(34);
		\draw[->] (27) --(35);
		\draw[->] (28) --(36);
		\draw[->] (29) --(36);
		\draw[->] (36) --(37);
		\draw[->] (38) --(27);
		
	\end{tikzpicture}
	\caption{The generic architecture for an NODE-ONet, which takes inputs consisting of $v\in \mathcal{V}$ and $x\in \Omega$, and outputs the value $\Psi_{\textnormal{NODE-ONet}}(v;\theta)(t,x)$ to approximate $\Psi^\dagger(v)(t,x)=u(t,x)$, for any $t\in [0,T]$.}
	\label{fig: deepNODE}
\end{figure}

Next, we provide two clarifying remarks on the pseudo-inverse properties of the encoder and the decoder, and the intuition behind the NODE component described in (NS2) of the NODE-ONet framework above.
\begin{remark}
	\textbf{(Pseudo-inverses).} By applying Lemma~\ref{lemma:inverse-stationary} in the stationary setting, we deduce that the encoder $E_{\mathcal{V}}$ and the decoder $D_{\mathcal{U}}$ (defined in (NS1) and (NS3) above) admit generalized inverses $D_{\mathcal{V}}$ and $E_{\mathcal{U}}$ in the sense of Assumption~\ref{ass1}(2). These inverses are explicitly given by
	\begin{align*}
		&D_{\mathcal{V}}(z)(t) = D_{\mathcal{V}}^s(z(t)), && \forall \, t \in [0,T], \; \forall\, z \in \mathcal{C}\big([0,T];\, \mathbb{R}^{d_{\mathcal{V}}}\big), \\
		&E_{\mathcal{U}}(u)(t) = E_{\mathcal{U}}^s(u(t,\cdot)), && \forall \, t \in [0,T], \; \forall\, u \in \mathcal{C}\big([0,T];\, \mathcal{C}(\Omega)\big),
	\end{align*}
	where $D_{\mathcal{V}}^s$ and $E_{\mathcal{U}}^s$ are defined in Lemma~\ref{lemma:inverse-stationary}.
\end{remark}
\begin{remark}\label{re:analysis}
	\textbf{(Neural ODE and semi-discrete schemes).} The NODE-induced neural network $\psi_{\theta}: \bm{v}\mapsto \bm{\psi}$, given in (NS2), aims to approximate the composite operator $\psi = E_\mathcal{U} \circ \Psi^\dagger \circ D_\mathcal{V}$. In contrast to the stationary case, establishing a universal approximation property for $\psi_{\theta}$ is generally more challenging, since $\psi$ does not uniformly map parameters of a vector field to the corresponding ODE  solutions . Nevertheless, $\psi$ admits an accurate approximation through a semi-discrete scheme applied to equation~\eqref{e_pde}.
	
	To illustrate, we consider learning the source-to-solution operator $\Psi_f \colon f \mapsto u$ for \eqref{e_pde}. Here, $v = f$, and we define the encoded input $f_h := E_{\mathcal{V}}(f) \in \mathcal{C}\big([0,T]; \mathbb{R}^{d_{\mathcal{V}}}\big)$ with $d_{\mathcal{U}} = d_{\mathcal{V}}$.  Then, $\psi(f_h)$ can be approximated by the solution $u_h\in\mathcal{C}([0,T];\, \mathbb{R}^{d_{\mathcal{U}}})$ of the semi-discrete version of~\eqref{e_pde}:
	\begin{equation}
		\frac{d u_h}{dt} + \mathcal{L}_h[a](u_h) = f_h, \quad \forall \, t \in [0,T],
	\end{equation}
	subject to appropriate initial and boundary conditions, where $\mathcal{L}_h[a]$ denotes the spatial discretization of $\mathcal{L}[a]$. The discrepancy between $\psi(f_h)$ and $u_h$ coincides with the semi-discretization error of \eqref{e_pde}, which is generally smaller and more tractable than that of a fully discrete scheme. In other words, the approximation error $\epsilon$ introduced by the NODE (see Assumption~\ref{ass1}(4) for the general setting) inherently includes the semi-discretization error. This is not problematic, since $\epsilon$ can be absorbed into the encoding-decoding errors without affecting the leading-order term of the total approximation error. A careful analysis of the overall error shapes the direction of subsequent works.
\end{remark}

\subsection{Training of NODE-ONets}

Given a dataset $\{v_i,x_j, \Psi^\dagger(v_i)(t,x_j)\}_{1\leq i\leq N_v, 1\leq j\leq N_x}$ of different input functions $\{v_i\}$ and spatial points $\{x_j\}$, the mean-squared error introduced by $\Psi_\text{NODE-ONet}$ reads
\begin{equation}\label{eq:loss_con}
	\frac{1}{N_vN_x}\sum_{i = 1}^{N_v}\sum_{j=1}^{N_x}\int_0^T\left\|\Psi_{\text{NODE-ONet}}(v_i;\theta)(t,x_j)-\Psi^\dagger(v_i)(t,x_j)\right\|_2^2 dt.
\end{equation}
To mitigate overfitting, we incorporate a regularization term $\mathcal{R}(\theta)$ into the mean-squared error to train $\Psi_{\text{NODE-ONet}}$. In practice, the integral in \eqref{eq:loss_con} is approximated by a quadrature scheme. To this end, we sample a set of temporal grid points $\{t_k\}_{k=1}^{N_t} \subset [0,T]$ and solve the NODE \eqref{eq:NODE} using a suitable ODE solver (e.g., Euler or Runge--Kutta methods) to compute $\{\Psi_{\text{NODE-ONet}}(v_i)(t_k,x_j)\}$ for $1 \leq i \leq N_v$, $1 \leq j \leq N_x$, and $1 \leq k \leq N_t$. The resulting loss function for training is then given by:
\begin{equation}\label{eq:loss_dis}
	\mathcal{L}(\theta)=\frac{1}{N_v N_xN_t}\sum_{i =1}^{N_v}\sum_{j=1}^{N_x}\sum_{k = 1}^{N_t}\|\Psi_{\text{NODE-ONet}}(v_i)(t_k,x_j)-\Psi^\dagger(v_i)(t_k,x_j)\|_2^2+\lambda\mathcal{R}(\theta),
\end{equation}
where $\lambda\geq0$ is a regularization parameter.

\begin{remark}
	The NODE-ONet framework can be implemented in a physics-informed machine learning manner. In particular, instead of \eqref{eq:loss_dis}, one can  minimize the residual of the underlying PDE in the spirit of the PINNs \cite{raissi2019physics}. 
\end{remark}

\begin{remark}
	The training inputs $\{v_i\}_{i=1}^{N_v}$ are usually drawn from a prescribed distribution. It is known that purely data-driven models, including learned operators, are predominantly effective in interpolation scenarios, i.e., when test inputs follow the training distribution, see e.g., \cite{barnard1992extrapolation,xu2021how,zhu2023reliable}. However,  real-world applications require extrapolation to out-of-distribution test inputs, which may lead to significant errors and model failure.
	To mitigate this extrapolation problem and improve the reliability of our model, we can integrate the approaches developed in \cite{zhu2023reliable} into the NODE-ONet framework.
\end{remark}

\section{Physics-Encoded NODEs}\label{se:pe_node}
The NODE-ONet framework is a high-level architecture that does not prescribe specific architectural details for its encoders and decoders. Consequently, a wide range of encoders and decoders including those proposed in \cite{choi2024spectral,hua2023basis,prasthofer2022variable,zhang2023belnet} can be directly integrated into this framework. Furthermore, with different NODEs, various NODE-ONets for PDEs can be specified from this framework. It is evident that the NODE component plays a pivotal role in determining the computational effectiveness and efficiency of the overall network. However, applying existing NODEs within NODE-ONets encounters various numerical challenges. To mitigate these limitations, we present a physics-encoded NODE design in this section and elaborate on its formulation and advantages with concrete examples.

Physics-encoded NODEs comprise two key components. First, the trainable parameters are either explicitly time-dependent or even time-invariant, reducing model complexity while enhancing temporal generalization. Second,  domain knowledge is encoded into the NODE architecture by leveraging the structural features inherent to the underlying PDEs. To make the discussions concrete, we focus on a nonlinear diffusion-reaction equation and a Navier-Stokes equation. 

\subsection{A nonlinear diffusion-reaction equation}
We specify \eqref{e_pde} as the following diffusion-reaction equation:
\begin{equation}\label{eq: diff-react}
	\left\{
	{	\begin{aligned}
			&\partial_t u(t,x)-\nabla\cdot(D(t,x)\nabla u(t,x))+R(t,x)u^2(t,x)=f(t,x), &\forall (t,x)\in [0,T]\times \Omega,\\
			&u(0,x)=u_0(x), & \forall x\in \Omega,\\
			&u(t,x)=u_b(t,x), &\forall (t,x)\in [0,T]\times \partial\Omega,
	\end{aligned}}
	\right.
\end{equation}
where $D: [0,T]\times \Omega\rightarrow \mathbb{R}$ is the diffusion coefficient, $R: [0,T]\times \Omega\rightarrow \mathbb{R}$ is the reaction coefficient, $f:[0,T]\times \Omega\rightarrow \mathbb{R}$ is the source term, $u_0(x)$ and $u_b(t,x)$ are the initial value and the boundary value, respectively. We consider $v:=\{D, R, f, u_0\}$ as the relevant input parameters for learning the solution operator of \eqref{eq: diff-react}.

The PDE (\ref{eq: diff-react}) combines diffusion, nonlinear reaction, and external forcing, leading to rich mathematical behavior and diverse applications. Analytical and numerical treatment depends critically on the interplay between these terms. In particular, we observe that the solution $u$ depends bilinearly on $D$ and nonlinearly on $R$, respectively. 

Inspired by the structural property of (\ref{eq: diff-react}), we propose the physics-encoded NODE:
\begin{equation}
\label{eq:NODE_pe3}
{	\begin{dcases}
\dot{\bm{\psi}}(t) =\sum_{i = 1}^P \Big\{(W_i\odot [\mathcal{P}_r\bm{R}(t)]+V_i)\odot\sigma\left(A_i\odot[\mathcal{P}_D\bm{D}(t)]\odot\bm{\psi} +A_i^n(t)+B_i\right)+\mathcal{P}_f\bm{f}(t)\Big\},\\
\bm{\psi}(0) = \mathcal{P}_u\bm{u}_0\in \mathbb{R}^{d_{\mathcal{U}}}.
\end{dcases}}
\end{equation}
Above, $\sigma:\mathbb{R}^{d_{\mathcal{U}}}\to\mathbb{R}^{d_{\mathcal{U}}}$ is an activation function, $\odot$ stands for the Hadamard product, $ W_i,V_i, A_i, B_i \in \mathbb{R}^{d_{\mathcal{U}}}$ for $i = 1, \ldots, P$ are \emph{time-independent} trainable parameters, $\mathcal{P}_D, \mathcal{P}_r, \mathcal{P}_f, \mathcal{P}_u\in \mathbb{R}^{{d_{\mathcal{U}}} \times {d_{\mathcal{V}}}}$ are trainable matrices, and $\bm{u}_0,\bm{D}(t), \bm{R}(t), \bm{f}(t)\in \mathbb{R}^{d_{\mathcal{V}}}$ are obtained via space discretization of $u_0(x)$, $D(t,x)$, $R(t,x),$ and $f(t,x)$, respectively. Furthermore, $A_i^n(t)\in \mathbb{R}^{d_{\mathcal{U}}}$ is assumed to be a
polynomial of degree $n$ with respect to $t$, which generally can be written in the form of 
\begin{equation}\label{eq:time_poly}
A_i^n(t)=\bm{a}_i^nt^n+\bm{a}_i^{n-1}t^{n-1}+\cdots+\bm{a}_i^1 t+\bm{a}_i^0
\end{equation}
with $\bm{a}_i^j\in\mathbb{R}^{d_{\mathcal{U}}}, 1\leq j\leq n, 1\leq i\leq P.$ 

It is notable that the intrinsic structural properties of the PDE \eqref{eq: diff-react} $-$ such as the bilinear coupling between $D$ and $u$, the nonlinear dependence of 
$R$ and $u$, and the additive source term $f$, are preserved in  \eqref{eq:NODE_pe3}, ensuring alignment with the underlying physical principles. This preservation  enables the NODE \eqref{eq:NODE_pe3} to effectively learn the dynamics of \eqref{eq: diff-react} and hence allows the resulting NODE-ONet to extrapolate solutions beyond the training temporal domain. 
Let $\{t_k\}_{k=1}^{N_t}\subset [0,T]$ be a set of grids for the temporal discretization of  \eqref{eq:NODE_pe3}. The architecture of the resulting NODE-ONet, designed to solve the reaction-diffusion PDE \eqref{eq: diff-react}, is illustrated in Figure \ref{fig:pe_NODE}, where  $\bm{\alpha}:=\{\alpha_j(x)\}_{j=1}^{d_{\mathcal{U}}}$ is supposed to the output of a neural network $\mathcal{N}_{\theta_\alpha}$.

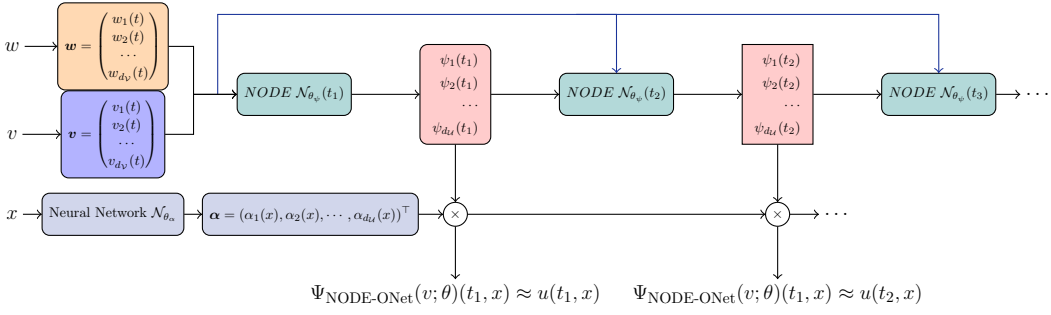
\begin{figure}[htpb]
\centering
\begin{tikzpicture}[global scale =0.53]
\centering

\node at(2.5,7)   (1) [font=\fontsize{15}{15}\selectfont]{$v$};
\node at(5,7)   (2) [zbox,rounded corners=1mm,fill=blue!30]{$\bm{v}=\begin{pmatrix}
	v_1(t)\\v_2(t)\\ \cdots\\v_{d_\mathcal{V}}(t)
\end{pmatrix}$};
\node at(2.5,9.2)   (18) [font=\fontsize{15}{15}\selectfont]{$w$};
\node at(5,9.2)   (19) [zbox,rounded corners=1mm, fill=orange!30]{$\bm{w}=\begin{pmatrix}
	w_1(t)\\w_2(t)\\ \cdots\\w_{d_\mathcal{V}}(t)
\end{pmatrix}$};
\node at(9.5,8) (3)[rectangle,rounded corners=1mm, minimum width =70pt, minimum height =30pt, inner sep=5pt,draw=black, fill=teal!30]  {$NODE~ \mathcal{N}_{\theta_{\psi}}(t_1)$}; 
\node at(13.5,8)   (4) [zbox,rounded corners=1mm, fill=red!20]{$\begin{aligned}
	\psi_1(t_1)\\\psi_2(t_1)\\ \cdots\\\psi_{d_{\mathcal{U}}}(t_1)
\end{aligned}$};
\node at(17.5,8) (5)[rectangle,rounded corners=1mm, minimum width =70pt, minimum height =30pt, inner sep=5pt,draw=black, fill=teal!30]  {$NODE~ \mathcal{N}_{\theta_{\psi}}(t_2)$}; 
\node at(21.5,8)   (6) [zbox, fill=red!20]{$\begin{aligned}
	\psi_1(t_2)\\\psi_2(t_2)\\ \cdots\\\psi_{d_{\mathcal{U}}}(t_2)
\end{aligned}$};
\node at(25.5,8) (7)[rectangle,rounded corners=1mm, minimum width =70pt, minimum height =30pt, inner sep=5pt,draw=black, fill=teal!30]  {$NODE~ \mathcal{N}_{\theta_{\psi}}(t_3)$}; 
\node at(28,8) (16)[font=\fontsize{15}{15}\selectfont,line width=1.2 pt]{$\cdots$};
\node at(2.5,5)   (8) [font=\fontsize{15}{15}\selectfont]{$x$};
\node at(5,5) (9)[rectangle,rounded corners=1mm, minimum width =70pt, minimum height =30pt, inner sep=5pt,draw=black, fill=myColor!20]  {Neural Network $\mathcal{N}_{\theta_{\alpha}}$};
\node at(9.9,5) (10)[rectangle,rounded corners=1mm, minimum width =90pt, minimum height =30pt, inner sep=5pt,draw=black, fill=myColor!20]  {$\bm{\alpha}=(\alpha_1(x),\alpha_2(x),\cdots,\alpha_{d_{\mathcal{U}}}(x))^\top$};
\node at(13.5,5)     (11)[circle, minimum width =17pt, minimum height =17pt, inner sep=0.1pt, draw=black]{$\times$};
\node at(21.5,5)     (12)[circle, minimum width =17pt, minimum height =17pt, inner sep=0.1pt, draw=black]{$\times$};
\node at(23,5) (15)[font=\fontsize{15}{15}\selectfont,line width=1.2 pt]{$\cdots$};
\node at(13.5,3)   (13) [font=\fontsize{15}{15}\selectfont]{$\Psi_{\text{NODE-ONet}}(v;\theta)(t_1,x)\approx u(t_1,x)$};
\node at(21.5,3)   (14) [font=\fontsize{15}{15}\selectfont]{$\Psi_{\text{NODE-ONet}}(v;\theta)(t_1,x)\approx u(t_2,x)$};
\draw[->](1)--(2);
\draw[->](3)--(4);
\draw[->](4)--(5);
\draw[->](5)--(6);
\draw[->](6)--(7);
\draw[->](8)--(9);
\draw[->](9)--(10);
\draw[->](10)--(11);
\draw[->](11)--(12);
\draw[->](12)--(15);
\draw[->](11)--(13);
\draw[->](12)--(14);
\draw[->](6)--(12);
\draw[->](4)--(11);
\draw[->](7)--(16);
\draw[->](18)--(19);
\draw[->][draw=myColor](7.6,8)--(7.6,10)--(17.5,10)--(5);
\draw[->][draw=myColor](16.5,10)--(25.5,10)--(7);
\draw[->](19)--(7,9.2)--(7,8)--(3);
\draw[->](2)--(7,7)--(7,8)--(3);
\end{tikzpicture}
\caption{A physics-encoded NODE-ONet for learning $\Psi^\dagger: v(t,x)\mapsto u(t,x)$ of \eqref{eq: diff-react}, where $ \mathcal{N}_{\theta_{\psi}}$ represents the NODE \eqref{eq:NODE_pe3} with $\theta_{\psi}$ the set of trainable parameters,  $v\subset\{D, f, R, u_0\}$ and $w=\{D, f, R, u_0\}\setminus v$ are the sets of input functions and known functions, respectively .}
\label{fig:pe_NODE}
\end{figure}

\subsection{A Navier-Stokes equation}
Let $\Omega=(0,1)^2$. For a given $T>0$, we consider the incompressible Navier-Stokes equation in the vorticity–velocity form:
\begin{equation}\label{eq:2dexample}
\left\{
\begin{aligned}
&\partial_t u(t, x) + \bm{V}(t, x)  \cdot \nabla u(t, x) = \nu \Delta u(t, x) + f(t, x) , \quad &\forall (t,x)\in [0,T]\times \Omega, \\
& u(t, x)  = \nabla\times \bm{V}(t, x)  \coloneqq \partial_{x_1} \bm{V}_2-\partial_{x_2} \bm{V}_1, \quad &\forall (t,x)\in [0,T]\times \Omega, \\
&\nabla \cdot \bm{V} (t, x) = 0, \quad &\forall (t,x)\in [0,T]\times \Omega, \\
&u(0,x) = u_0(x), \quad &\forall x \in \Omega,
\end{aligned}
\right. 
\end{equation}
with proper boundary conditions. 
In \eqref{eq:2dexample}, $\bm{V}(t, x)$ and $u(t, x)$ are the velocity and  the vorticity, respectively; $\nu>0$ is the viscosity coefficient, $u_0(x)$ is the initial vorticity, and $f(t, x) $ is the forcing term. 

It follows from the linearity of $\nabla$ and $\nabla\times$ that the term $ \bm{V}(t, x)  \cdot \nabla u(t, x) $ in \eqref{eq:2dexample} can be viewed as a quadratic form $\mathcal{F}u\cdot\mathcal{G}u$ with linear operators $\mathcal{F}$ and $\mathcal{G}$ to be determined. Moreover, the forcing term $f$ is additive to the equation. 
Inspired by these observations, we consider the following physics-encoded NODE:
\begin{equation}
\label{eq:NODE_ns}
{\small	\begin{dcases}
\dot{\bm{\psi}}(t) =\sum_{i = 1}^P \Big\{W_i\odot\sigma\Big(A_i\odot\bm{\psi}+[C_i\odot\bm{\psi}]\odot [D_i\odot\bm{\psi}] +A_i^n(t)+B_i\Big)+\mathcal{P}_f\bm{f}\Big\},\\
\bm{\psi}(0) = \mathcal{P}_u\bm{u}_0\in \mathbb{R}^{d_{\mathcal{U}}},
\end{dcases}}
\end{equation}
where $\sigma:\mathbb{R}^{d_{\mathcal{U}}}\to\mathbb{R}^{d_{\mathcal{U}}}$ is an activation function, $\odot$ stands for the Hadamard product, $ W_i, A_i, B_i, C_i, D_i \in \mathbb{R}^{d_{\mathcal{U}}}$ for $i = 1, \ldots, P$ are \emph{time-independent} trainable parameters, $\mathcal{P}_f, \mathcal{P}_u\in \mathbb{R}^{{d_{\mathcal{U}}} \times {d_{\mathcal{V}}}}$ are trainable matrices, and $\bm{u}_0,  \bm{f}\in \mathbb{R}^{d_{\mathcal{V}}}$ are obtained via space discretization of $u_0(x)$ and $f(x)$, respectively. Furthermore, $A_i^n(t)\in \mathbb{R}^{d_{\mathcal{U}}}$ is given in \eqref{eq:time_poly}.  
The NODE \eqref{eq:NODE_ns} possesses the same desirable properties as \eqref{eq:NODE_pe3}. The resulting NODE-ONet architecture is analogous to the one shown in Figure \ref{fig:pe_NODE}, and we therefore omit a detailed description.

Thanks to  the physics-encoded NODEs \eqref{eq:NODE_pe3} and \eqref{eq:NODE_ns}, the corresponding NODE-ONets explicitly incorporate the influence of known system-specific functions, a feature that distinguishes it from conventional operator learning methodologies. This integration not only enhances numerical efficiency but also improves predictive accuracy.
To further elucidate the implementation of physics-encoded NODEs and validate the advantages of the resulting NODE-ONets across diverse contexts, we present some case studies in the following section.

\section{Applications and Numerical Simulations}\label{se:num}

In this section, we validate the effectiveness, efficiency, and flexibility of the proposed physics-encoded NODEs and the resulting NODE-ONets. To this end, we focus on the nonlinear diffusion-reaction equation \eqref{eq: diff-react} and the Navier-Stokes equation \eqref{eq:2dexample}. Some numerical comparisons with state-of-the-art operator learning methods are also included.  All codes used in this paper are written
in Python and PyTorch and are publicly available on \url{https://github.com/DCN-FAU-AvH/NODE-ONet}.

\subsection{The nonlinear diffusion-reaction equation (\ref{eq: diff-react})}\label{se:num_DR}

We specify \eqref{eq: diff-react} as
\begin{equation}\label{eq:diff-react_num}
\partial_t u(t,x)-\nabla\cdot(D(x)\nabla u(t,x))+Ru^2(t,x)=f(x),\quad \forall (t,x)\in [0,1]\times [0,1],\\
\end{equation}
with zero initial and boundary conditions. This example has been well-studied in the literature of operator learning for PDEs, see e.g., \cite{jin2022mionet,lu2021learning}. We follow the above references and assume $D$ and $f$ are time-independent and the reaction coefficient $R=-0.01$. 

The purpose of considering the current example is three-fold. First, we implement the physics-encoded NODE \eqref{eq:NODE_pe3} to learn the solution operator $\Psi^\dagger: v(t,x)\mapsto u(t,x)$ of \eqref{eq:diff-react_num} in three input settings: $v=f(x)$, $v=D(x)$, and $v=(D(x),f(x))$, and we assess the effectiveness of the resulting NODE-ONets. Second, by comparing with some state-of-the-art operator networks including the DeepONets \cite{lu2021learning} and the MIONet \cite{jin2022mionet}, we shall show the advances of the NODE-ONets in terms of numerical efficiency and accuracy.   Finally, we test the generalization and prediction capabilities of the NODE-ONets.

\subsubsection{Set-ups}

To build the training set, we first generate a high-resolution reference dataset. We sample input functions $\{v_i\}_{i=1}^{N_{v_\text{train}}}$ from a Gaussian process $\mathcal{GP}(0,C)$ on the spatial domain,
\begin{equation}\label{s1}
v_i \sim \mathcal{GP}\!\left(0,\,C\right), \qquad 
C(x_1,x_2)=\exp\!\left(-\|x_1-x_2\|_2^2/(2l^2)\right),
\end{equation}
where $l>0$ is the length-scale of the Gaussian covariance kernel; larger $l$ yields smoother samples $v_i$. For each $v_i$, we solve \eqref{eq:diff-react_num} on a fine spatial grid of 1001 equi-spaced points using a finite-difference method to obtain the corresponding high-resolution solution $u_i$. To assemble training data at a target spatial resolution $N_x$, we define grid points $\{x_j\}_{j=1}^{N_x}$ and evaluate $u_i$ at these locations (via interpolation when needed), yielding tuples $(v_i,x_j,u_i(x_j))$. The collection $\{(v_i,x_j,u_i(x_j))\}_{1\le i\le N_{v_\text{train}},\;1\le j\le N_x}$ constitutes the training dataset for learning the solution operator $\Psi^\dagger: v\mapsto u$. Given each input function $v_i$, we follow the idea of DeepONets and define the encoder $E_{\mathcal{V}}: \mathcal{V} \to [0,T]\times\mathbb{R}^{d_{\mathcal{V}}}$ as $E_{\mathcal{V}}(v)=\bm{v}(t):=\{v_\ell(t)\}_{\ell=1}^{d_{\mathcal{V}}}\in \mathbb{R}^{d_{\mathcal{V}}}$ for any $t\in[0,T]$, where $v_\ell(t)=v(t,x_\ell)$ with $\{x_\ell\}_{\ell=1}^{d_{\mathcal{V}}}\subset\Omega$ a set of fixed sensors.

The neural network $\mathcal{N}_{\theta_\alpha}$ is set as an FCNN consisting of 2 hidden layers with $P=100$ neurons per layer and equipped with \texttt{ReLU} activation functions. We set $A_i^n(t)=\bm{a}_i^1t$ in \eqref{eq:NODE_pe3}.  All the NODEs use \texttt{ReLU} activation functions and are solved by the explicit Euler method with $N_t$ time steps.  The output dimension of $\mathcal{N}_{\theta_\alpha}$ and NODEs is set as $d_{\mathcal{U}}=50$.
Unless otherwise specified,  we minimize the loss function  \eqref{eq:loss_dis} with $\lambda=0$ by an ADAM optimizer with learning rate $10^{-3}$  for $1\times 10^5$ epochs to train the NODE-ONets,
The parameters for $\mathcal{N}_{\theta_\alpha}$ and the NODEs  are initialized by the default PyTorch settings.

To validate the test accuracy, we set $N_x=N_t=100$ and generate $N_{v_\text{test}}$ new input functions $v$ for the trained NODE-ONets following the same procedures for generating the training sets as the default setting.
To measure the test errors, we employ two kinds of metrics as follows:		
\begin{equation}\label{eq:errors_def}
\left\{
\begin{aligned}
&\text{Absolute error}:=\left(\frac{1}{N_{v_\text{test}} N_xN_t}\sum_{i =1}^{N_{v_\text{test}}}\sum_{j=1}^{N_x}\sum_{k = 1}^{N_t}\|\Psi_{\text{NODE-ONet}}(v_i)(t_k,x_j)-\Psi^\dagger(v_i)(t_k,x_j)\|^2_2\right)^{\frac{1}{2}},\\
&\text{Relative error}:=\text{Absolute error }/\left(\frac{1}{N_{v_\text{test}} N_xN_t}\sum_{i =1}^{N_{v_\text{test}}}\sum_{j=1}^{N_x}\sum_{k = 1}^{N_t}{\|\Psi^\dagger(v_i)(t_k,x_j)\|^2_2}\right)^\frac{1}{2}.
\end{aligned}
\right.
\end{equation}

\subsubsection{Results and discussions}
In this subsection, we elaborate on the implementation of the NODE-ONets with the physics-encoded NODE \eqref{eq:NODE_pe3} for solving \eqref{eq:diff-react_num} and present some preliminary results with various types of input functions. 

\noindent$\bullet$ \textbf{Learn the solution operator of \eqref{eq:diff-react_num} with a single input function.} We first learn the source-to-solution operator $\Psi_f^\dagger: v=f\mapsto u$. For this purpose, we take $D(t,x)=0.01$ in \eqref{eq:diff-react_num}, and the input function $f$ is generated by (\ref{s1}) with $l=0.5$. Since the coefficients $D$ and $R$ are assumed to be constants, the  physics-encoded NODE \eqref{eq:NODE_pe3} reduces to the following form:		
\begin{equation}
\label{eq:NODE_s}
\begin{dcases}
\dot{\bm{\psi}}(t) =\sum_{i = 1}^P W_i\odot\sigma(A_i\odot\bm{\psi}+\bm{a}_i^1t+B_i)+\mathcal{P}_f\bm{f},\\
\bm{\psi}(0) =\bm{0}\in \mathbb{R}^{d_{\mathcal{U}}}.
\end{dcases}
\end{equation}

We set $d_{\mathcal{V}}=20$, and take different $N_x$, $N_t$, and $N_{v_\text{train}}$ to train the NODE-ONet. We then sample $N_{v_\text{test}}$ new input functions $f(x)$ to test the learned operator $\Psi_f^*$. The test absolute and relative errors for different parameter settings are presented in Table \ref{tab:ex1_compare_deeponet}. The test results for one random input function $f(x)$ are reported in Figure \ref{fig: source_ex1}.

\begin{table}[h]
\setlength{\tabcolsep}{0.5em}
\scriptsize
\centering
\caption{Comparisons of the NODE-ONet with the unstacked DeepONet for learning the source-to-solution operator $\Psi_f^\dagger: f(x)\mapsto u(t,x)$ of (\ref{eq:diff-react_num}).}
\begin{tabular}[c]{|c||c|c|c|c|c|c|c|c|}
\hline
~ & \text{Training} & \text{Training}&$\#$Trainable &$\#$Training & Test& $\#$Test & Absolute & Relative \\
& \text{epochs} & \text{resolutions}&parameters& input $f$ & resolutions&input $f$ &error& error\\
\hline
\text{NODE-ONet}& $\begin{tabular}[c]{@{}l@{}}ADAM \\$5\times10^5$\end{tabular}$  & \begin{tabular}[c]{@{}l@{}}$N_x=10$\\ $N_t=5$\\$d_{\mathcal{V}}=20$\end{tabular} &27,550 &100&$\begin{tabular}[c]{@{}l@{}}$N_x=100$\\ $N_t=100$\end{tabular}$& 10,000 &$4.248 \times 10^{-3}$ & $ 7.370\times 10^{-3}$  \\
\hline

\text{NODE-ONet}& $\begin{tabular}[c]{@{}l@{}}ADAM \\$5\times10^5$\end{tabular}$  & \begin{tabular}[c]{@{}l@{}}$N_x=100$\\ $N_t=10$\\$d_{\mathcal{V}}=20$\end{tabular}&27,550 &500&$\begin{tabular}[c]{@{}l@{}}$N_x=100$\\ $N_t=100$\end{tabular}$& 10,000 &$1.368 \times 10^{-3}$ & $2.675 \times 10^{-3}$  \\
\hline
\text{DeepONet}  & $\begin{tabular}[c]{@{}l@{}}ADAM\\ $5\times10^5$\end{tabular}$  &$\begin{tabular}[c]{@{}l@{}}$N_x=100$\\ $N_t=10$\\$~K = 50$\\$d_{\mathcal{V}}=100$\end{tabular}$ &40,600& 100& $\begin{tabular}[c]{@{}l@{}}$N_x=100$\\ $N_t=100$\\$~K= 100$\end{tabular}$ & 10,000 &$6.352 \times 10^{-3}$ &$1.230 \times 10^{-2}$ \\
\hline
\text{DeepONet}  & $\begin{tabular}[c]{@{}l@{}}ADAM\\ $5\times10^5$\end{tabular}$  &$\begin{tabular}[c]{@{}l@{}}$N_x=100$\\ $N_t=100$\\$~K$ = 1,000\\$d_{\mathcal{V}}=100$\end{tabular}$ &40,600& 500& $\begin{tabular}[c]{@{}l@{}}$N_x=100$\\ $N_t=100$\\~$K$ = 1,000\end{tabular}$ & 10,000 &$1.313 \times 10^{-3}$ &$2.582 \times 10^{-3}$ \\
\hline
\end{tabular}
\label{tab:ex1_compare_deeponet}
\normalsize
\end{table}

For numerical comparisons, we implement the unstacked DeepONet \cite{lu2021learning} following the same setups therein. In particular, for the training of DeepONet, we randomly select $K$ points for each $f$ as the input of the trunk net.  The errors are listed in Table \ref{tab:ex1_compare_deeponet} and the numerical results for one random input function $f(x)$ are presented in Figure \ref{fig: DeepOnet_source_ex1}. The comparison results in Table \ref{tab:ex1_compare_deeponet} demonstrate that the  NODE-ONet achieves comparable or superior accuracy to the DeepONet in learning the source-to-solution operator \(\Psi_f\), while exhibiting significantly greater efficiency. In particular, the NODE-ONet attains low errors with coarser training resolutions  and fewer model parameters, whereas the DeepONet requires finer discretization and higher complexity to achieve similar accuracy. This validates the NODE-ONet's numerical efficiency in learning the solution operator of PDEs. 

\begin{figure}[h]
\caption{Test results of the NODE-ONet for learning the source-to-solution operator $\Psi_f^\dagger: f(x)\mapsto u(t,x)$ of \eqref{eq:diff-react_num} with one random $f(x)$. Training parameters: $N_x=100,N_t=10, N_{v_\text{train}}=500$. Test parameters: $N_x=N_t=100, N_{v_\text{test}}=10,000$. }\label{fig: source_ex1}
\centering
\includegraphics[scale=0.3]{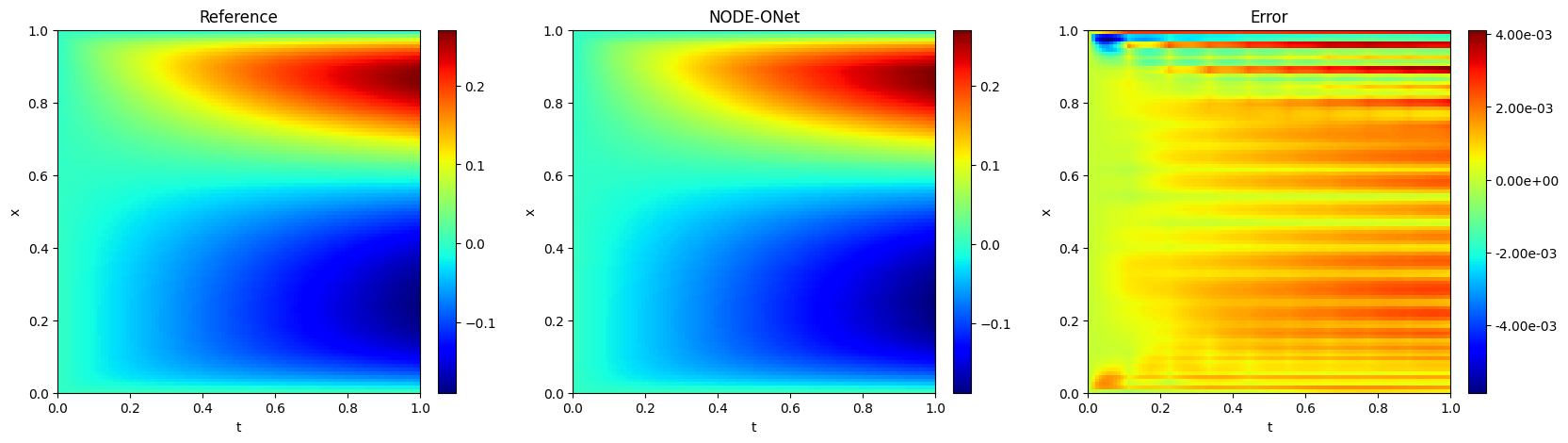}
\end{figure}

\begin{figure}[h]
\caption{ Test results of the DeepONet for learning the source-to-solution operator $\Psi_f^\dagger: f(x)\mapsto u(t,x)$ of \eqref{eq:diff-react_num} with one random $f(x)$. Training: $N_x=100,N_t=100, N_{v_\text{train}}=100$. Test: $N_x=N_t=100, N_{v_\text{test}}=10,000$.}\label{fig: DeepOnet_source_ex1}
\centering
\includegraphics[scale=0.3]{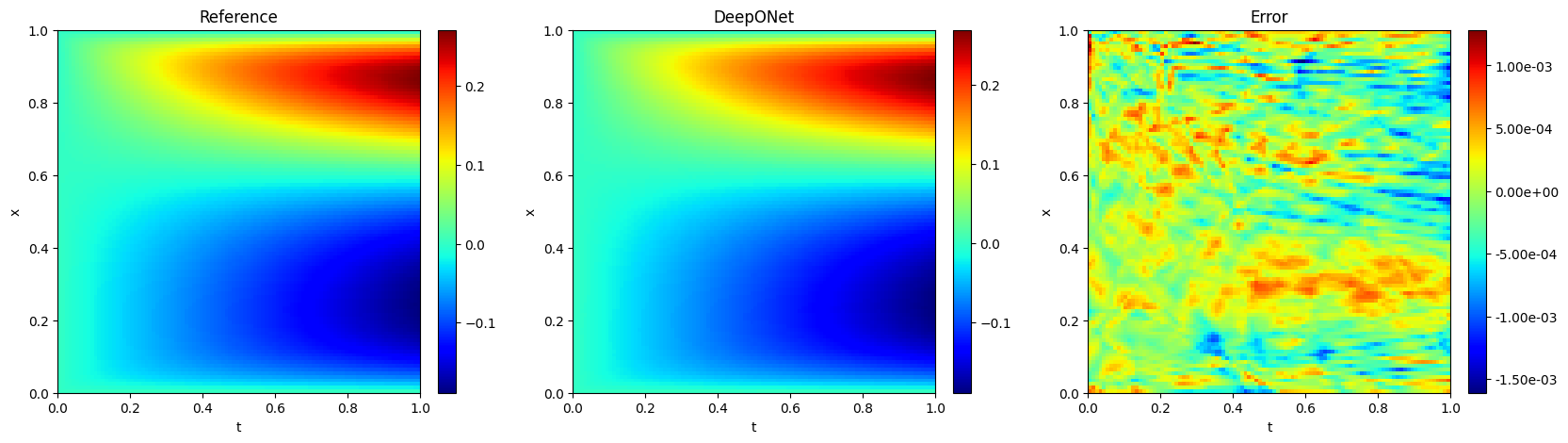}
\end{figure}

We then consider the diffusion-to-solution operator $\Psi_D^\dagger: v=D\mapsto u$. We take $f(x)=\sin (2\pi x)$ in \eqref{eq:diff-react_num}, and set $l=0.5$ in \eqref{s1} to generate the input function $D(x)$. In this case, the  physics-encoded NODE \eqref{eq:NODE_pe3} is in the form of \eqref{eq:NODE_d}. It is notable that $\bm{f}:=\{f(x_\ell)\}_{\ell=1}^{d_\mathcal{V}}$ in  \eqref{eq:NODE_d} is a spatial discretization of the given  $f(x)$. 
\begin{equation}
\label{eq:NODE_d}
\begin{dcases}
\dot{\bm{\psi}}(t) =\sum_{i = 1}^P W_i\odot\sigma(A_i\odot[\mathcal{P}_D\bm{D}]\odot\bm{\psi}+\bm{a}_i^1t+B_i)+\mathcal{P}_f\bm{f},\\
\bm{\psi}(0) =\bm{0}\in \mathbb{R}^{d_{\mathcal{U}}}.
\end{dcases}
\end{equation}

For the training of the NODE-ONet, we take $N_x=100,N_t=10$, $d_{\mathcal{V}}=20$, and $N_{v_\text{train}}=1,000$. Then, we  set $N_x=N_t=100$ and sample $N_{v_\text{test}}=10,000$ new input functions $D(x)$ to test the learned operator $\Psi_D^*$. The test absolute and relative errors are respectively $1.276\times 10^{-3}$ and $3.919\times10^{-3}$ and the test results for one random input function $D(x)$ are reported in Figure \ref{fig: diffusion_ex1}. As expected, the NODE-ONet is effective in learning the diffusion-to-solution operator. 
\begin{figure}[htpb]
\caption{ Test results of the NODE-ONet for learning the diffusion-to-solution operator $\Psi_D^\dagger: D(x)\mapsto u(t,x)$ of \eqref{eq:diff-react_num} with one random $D(x)$.}\label{fig: diffusion_ex1}
\centering
\includegraphics[scale=0.3]{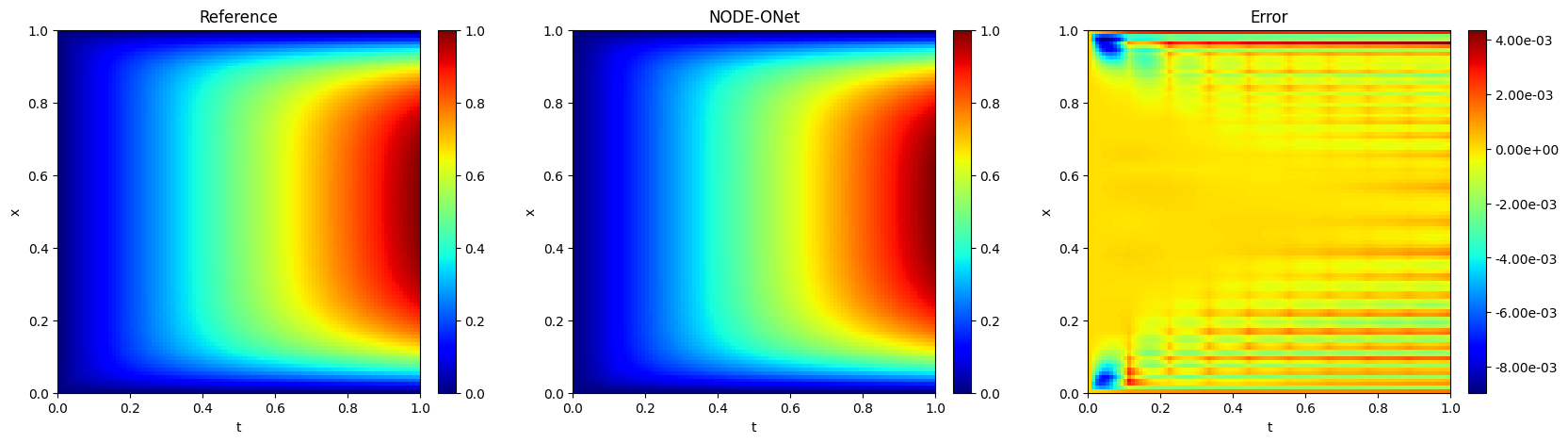}
\end{figure}

\noindent$\bullet$ \textbf{Learn the solution operator of \eqref{eq:diff-react_num} with multi-input functions.} Next, we apply the NODE-ONet to learn the solution operator of \eqref{eq:diff-react_num} with multi-inputs $\Psi_m^\dagger: v=\{D,f\}\mapsto u$. We follow \cite{jin2022mionet} to set $D(x)=0.01(|g(x)|+1)$ and the functions $g(x)$ and $f(x)$ are generated by (\ref{s1}) with $l=0.2$. It follows from \eqref{eq:NODE_pe3} that the  physics-encoded NODE used here reads
\begin{equation}
\label{eq:NODE_m}
\begin{dcases}
\dot{\bm{\psi}}(t) =\sum_{i = 1}^P W_i\odot\sigma(A_i\odot[\mathcal{P}_D\bm{D}]\odot\bm{\psi}+\bm{a}_i^1t+B_i)+\mathcal{P}_f\bm{f},\\
\bm{\psi}(0) = \bm{0}\in \mathbb{R}^{d_{\mathcal{U}}}.
\end{dcases}
\end{equation}

We set $d_{\mathcal{V}}=20$ and take different $N_x, N_t, N_{v_\text{train}}$ to train the NODE-ONet. To validate the effectiveness and efficiency, we compare the NODE-ONet with the MIONet \cite{jin2022mionet}, which is one of the benchmark algorithms for learning operators with multi-input functions.  The MIONet consists of two branch nets for encoding the input functions $D$ and $f$ and a trunk net for encoding the domain $(t, x)$ of the output function $u$. We implement the MIONet following the settings given in \cite{jin2022mionet}.

The test results of the NODE-ONet and the MIONet are listed in Table \ref{tab:ex1_compare_minonet} and the numerical results for one random input pair $\{D(x),f(x)\}$ are reported in Figures \ref{fig:multi-input_ex1}  and  \ref{fig:MINOnet_multi-input_ex1}. The NODE-ONet demonstrates superior efficiency and accuracy over the MIONet when learning the multi-input solution operator \(\Psi_m^\dagger\). With only 100 training pairs at moderate resolution (\(N_x=50\), \(N_t=10\)), the NODE-ONet achieves significantly lower errors compared to the MIONet trained at higher resolution. For the case of $N_{v_\text{train}}=1,000$, the NODE-ONet with coarser temporal resolution still outperforms the MIONet with finer discretization. These results highlight NODE-ONet's advantages in computational efficiency for operator learning with multiple inputs.

\begin{table}[h]
\setlength{\tabcolsep}{0.5em}
\scriptsize
\centering
\caption{Comparisons of the NODE-ONet with the MIONet for learning the solution operator of \eqref{eq:diff-react_num} with multi-input functions: $\Psi_m^\dagger: \{D(x),f(x)\}\mapsto u(t,x)$.}
\begin{tabular}[c]{|c||c|c|c|c|c|c|c|c|}
\hline
~ & \text{Training} & \text{Training}& $\#$Trainable& $\#$Training & Test&$\#$Test &Absolute & Relative \\
& \text{epochs} & \text{resolutions}&parameters& $ \{D,f\}$ & resolutions& $ \{D,f\}$ &error& error\\
\hline
\text{NODE-ONet}&  $\begin{tabular}[c]{@{}l@{}}ADAM\\ $1\times10^5$\end{tabular}$ & \begin{tabular}[c]{@{}l@{}}$N_x=50$\\ $N_t=10$\end{tabular}&28,550 & 100&$\begin{tabular}[c]{@{}l@{}}$N_x=100$\\ $N_t=100$\end{tabular}$& 5,000 &$2.362 \times 10^{-2}$ & $5.297 \times 10^{-2}$  \\
\hline
\text{NODE-ONet}&  $\begin{tabular}[c]{@{}l@{}}ADAM\\ $1\times10^5$\end{tabular}$  & \begin{tabular}[c]{@{}l@{}}$N_x=100$\\ $N_t=10$\end{tabular}&28,550 & 1,000&$\begin{tabular}[c]{@{}l@{}}$N_x=100$\\ $N_t=100$\end{tabular}$& 5,000 &$4.626 \times 10^{-3}$ & $1.032 \times 10^{-2}$  \\
\hline
\text{MIONet}  &  $\begin{tabular}[c]{@{}l@{}}ADAM\\ $1\times10^5$\end{tabular}$  &$\begin{tabular}[c]{@{}l@{}}$N_x=100$\\ $N_t=100$\end{tabular}$&161,600&100& $\begin{tabular}[c]{@{}l@{}}$N_x=100$\\ $N_t=100$\end{tabular}$ & 5,000 &$1.212 \times 10^{-1}$ & $2.661 \times 10^{-1}$ \\
\hline
\text{MIONet}  &  $\begin{tabular}[c]{@{}l@{}}ADAM\\ $1\times10^5$\end{tabular}$  &$\begin{tabular}[c]{@{}l@{}}$N_x=100$\\ $N_t=100$\end{tabular}$&161,600&1,000& $\begin{tabular}[c]{@{}l@{}}$N_x=100$\\ $N_t=100$\end{tabular}$ & 5,000 &$9.491 \times 10^{-3}$ & $2.072 \times 10^{-2}$ \\
\hline

\end{tabular}
\label{tab:ex1_compare_minonet}
\normalsize
\end{table}

\begin{figure}[htpb]
\caption{ Test results of the NODE-ONet for learning the solution operator $\Psi_m^\dagger:\{D(x),f(x)\}\mapsto u(t,x)$ of \eqref{eq:diff-react_num} with one random multi-input pair. Training: $N_x=100,N_t=10, N_{v_\text{train}}=1000$. Test: $N_x=N_t=100, N_{v_\text{test}}=5000$. }\label{fig:multi-input_ex1}
\centering
\includegraphics[scale=0.3]{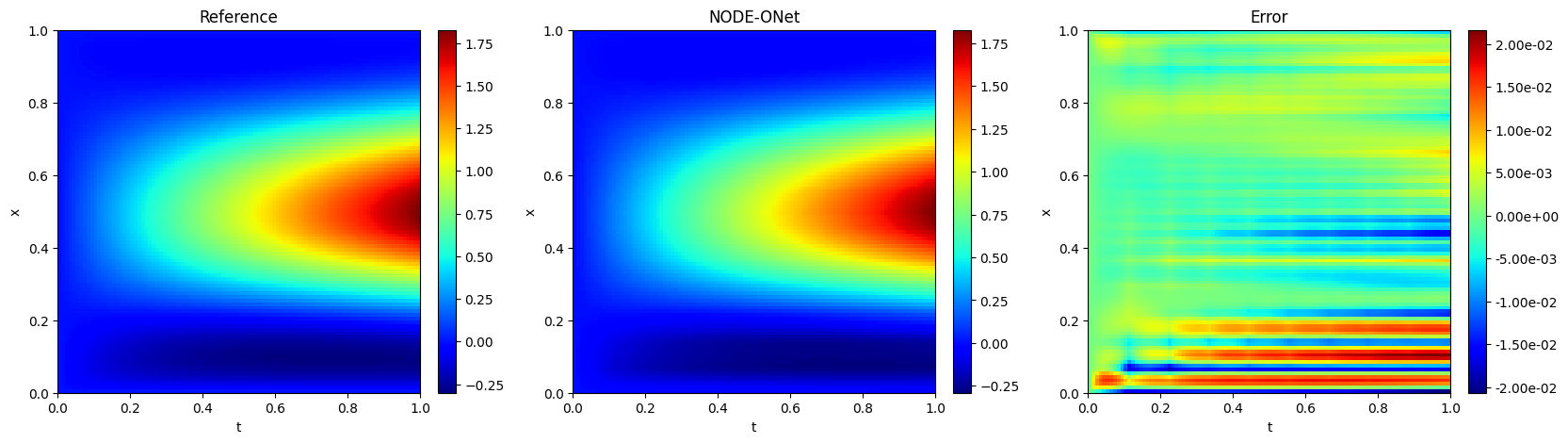}
\end{figure}

\begin{figure}[!htbp]
\caption{ Test results of the MIONet for learning the solution operator $\Psi_m^\dagger: \{D(x),f(x)\}\mapsto u(t,x)$ of \eqref{eq:diff-react_num} with one random input pair. Training: $N_x=N_t=100, N_{v_\text{train}}=1000$. Test: $N_x=N_t=100, N_{v_\text{test}}=5000$. }\label{fig:MINOnet_multi-input_ex1}
\centering
\includegraphics[scale=0.3]{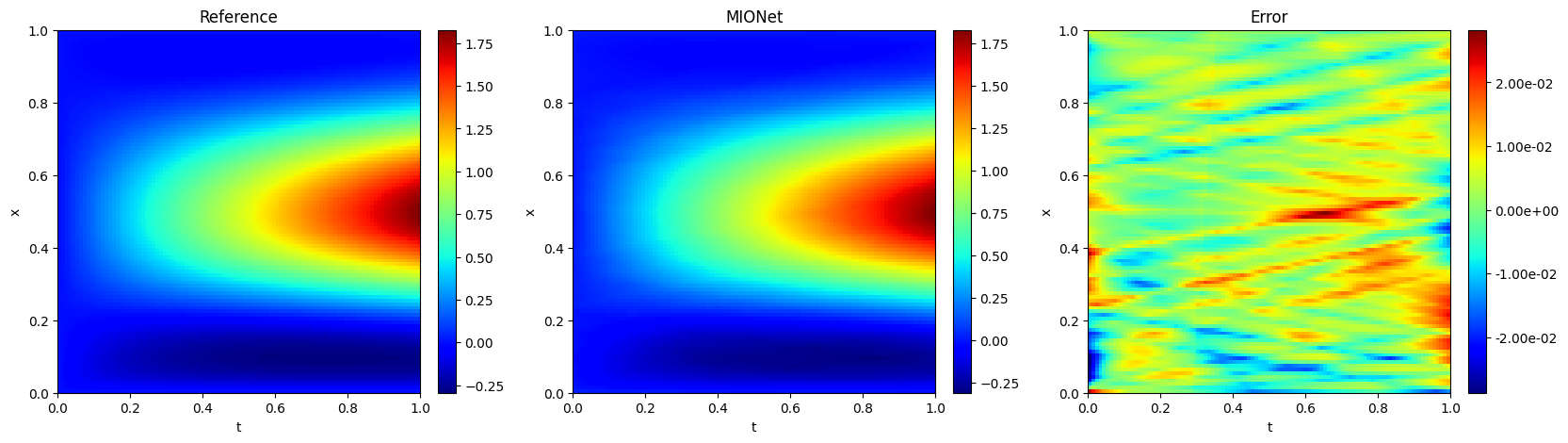}
\end{figure}

\noindent$\bullet$ \textbf{Generalization capacity of $\bm{\alpha}$.} Thanks to the separation of temporal and spatial variables in the NODE-ONet, we expect that the trained neural network $N_{\theta_{\alpha}^*}$ (and hence the corresponding output $\bm{\alpha}^*$) can be directly used as a decoder for learning other PDEs with similar structures. For validation, we first set $D=0.01, R=-0.01$ in \eqref{eq:diff-react_num} and learn the source-to-solution operator $\Psi_f^\dagger$ by the NODE-ONet. We set $l=0.5$ in \eqref{s1} to generate a training set with 500 input functions to get the trained neural network $N_{\theta_{\alpha}^*}$. 

Then, we set $D=0.2$, $R=0$ and hence the PDE \eqref{eq:diff-react_num} reduces to a diffusion equation. We aim to learn the resulting source-to-solution operator by the NODE-ONet with the pre-trained neural network $N_{\theta_{\alpha}^*}$. Note that, in this case, one only needs to train the NODE \eqref{eq:NODE_s}. For this purpose, we sample 500 input functions from \eqref{s1} with $l=0.3$ to generate a training set. To test the numerical accuracy, we set $l=0.3$ in \eqref{s1} to generate 10,000 new input functions. The test absolute and relative errors are $1.836\times 10^{-3}$ and $7.670\times 10^{-3}$, respectively. The numerical results with one random input function $f$ are presented in Figure \ref{fig: space_generalization_ex1}. 
\begin{figure}[!htpb]
\caption{ Test results for learning the source-to-solution operator $\Psi_f^\dagger: f(x)\mapsto u(t,x)$ of \eqref{eq:diff-react_num} with trained $\bm{\alpha}$. 
}\label{fig: space_generalization_ex1}
\centering
\includegraphics[scale=0.3]{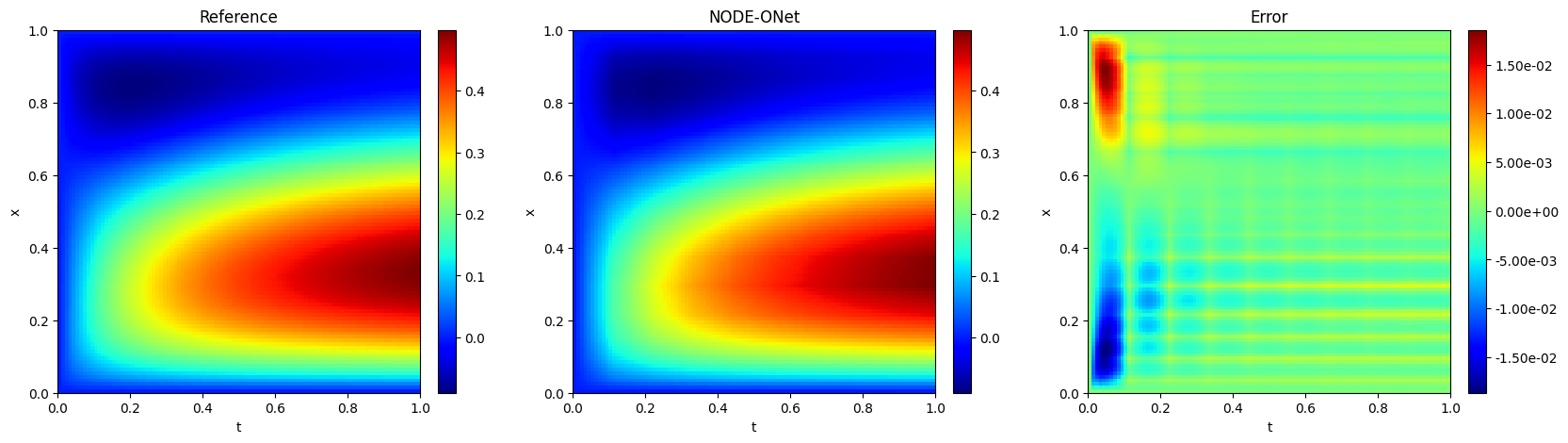}
\end{figure}

\noindent$\bullet$ \textbf{Prediction beyond the training time frame.} We test the above learned source-to-solution operator $\Psi_f^*$ and the operator $\Psi_m^*$ with multi-input functions for $t\in [0, 2]$. Note that these two operators are trained for $t\in[0,1]$. The test errors are listed in Table \ref{tab:ex1_prediction}, from which we can see that, compared with the DeepONet and MIONet, the NODE-ONets achieve much higher numerical accuracy for $t\in[0,2]$.

The numerical results with one random input function/pair are reported in Figures \ref{fig: prediction_time_ex1} -- \ref{fig: prediction_mionet_ex1}. We observe that the test errors for all methods are small for $t\in [0,1]$.  However, the prediction errors for $t\in[1,2]$ of the NODE-ONets are much smaller than those of the DeepONet and MIONet. 
These results demonstrate that the proposed physics-encoded NODE \eqref{eq:NODE_pe3} can efficiently learn the evolution pattern of the PDE (\ref{eq:diff-react_num}). Hence, the resulting NODE-ONet is capable of predicting the system beyond the training time frame, distinguishing it from DeepONets and MIONet.

\begin{table}[h]
\scriptsize
\centering
\caption{Prediction results of \eqref{eq:diff-react_num} for $t\in[0,2]$.  $\Psi_f^*:f(x)\mapsto u(t,x)$: the learned source-to-solution operator  ,$\Psi_m^*: \{D(x),f(x)\}\mapsto u(t,x)$: the learned solution operator with multi-input functions.}\label{tab:ex1_prediction}
\begin{tabular}{|cc|c|c|c|c|c|}
\hline
\multicolumn{2}{|l|}{\multirow{2}{*}{}}                       & $\#$Training&Training                     & Test                         & Absolute               & Relative               \\  
\multicolumn{2}{|l|}{}                                        & input functions&time frame                   & time frame                   & error            & error            \\ \hline
\multicolumn{1}{|c|}{\multirow{2}{*}{$\Psi_f^*$}} & NODE-ONet & \multirow{2}{*}{500}&\multirow{2}{*}{$t\in[0,1]$} & \multirow{2}{*}{$t\in[0,2]$} & $6.839 \times 10^{-3}$ & $7.113 \times 10^{-3}$ \\ \cline{2-2} \cline{6-7} 
\multicolumn{1}{|c|}{}                            & DeepONet  &                              &          &                    & $2.302\times10^{-1}$   & $2.360 \times 10^{-1}$ \\ \hline
\multicolumn{1}{|l|}{\multirow{2}{*}{$\Psi_m^*$}} & NODE-ONet & \multirow{2}{*}{1,000}&\multirow{2}{*}{$t\in[0,1]$} & \multirow{2}{*}{$t\in[0,2]$} & $1.392 \times 10^{-2}$ & $1.732 \times 10^{-2}$ \\ \cline{2-2} \cline{6-7} 
\multicolumn{1}{|l|}{}                            & MIONet    &                              &         &                     & $1.012 \times 10^{-1}$ & $1.251 \times 10^{-1}$ \\ \hline
\end{tabular}
\normalsize
\end{table}

\begin{figure}[!htbp]
\caption{Prediction solution of \eqref{eq:diff-react_num} by the NODE-ONet for the learned source-to-solution operator $\Psi_f^*: f(x)\mapsto u(t,x)$ with one random input function. (Test time frame $t\in [0, 2]$; Training time frame $t\in [0,1]$)
}\label{fig: prediction_time_ex1}
\centering
\includegraphics[width=1\textwidth]{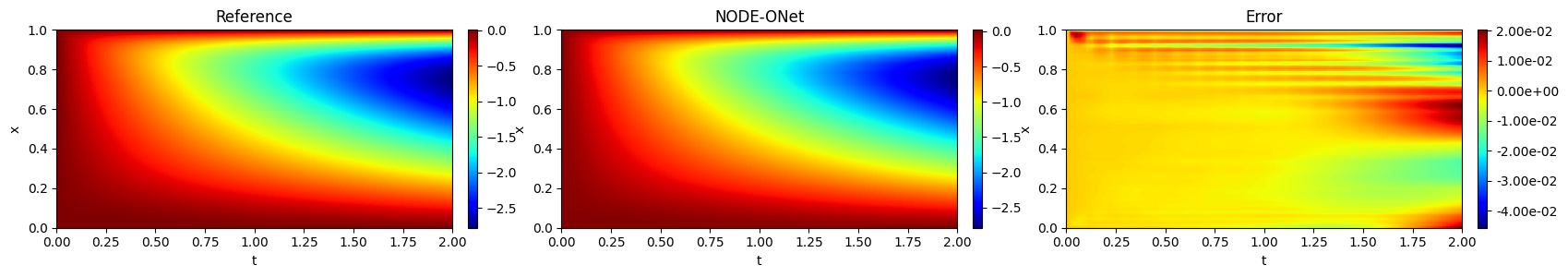}\\
\end{figure}

\begin{figure}[!htbp]
\caption{ Prediction solution of \eqref{eq:diff-react_num} by the DeepONet for the learned source-to-solution operator $\Psi_f^*: f(x)\mapsto u(t,x)$ with one random input function. (Test time frame $t\in [0, 2]$; Training time frame $t\in [0,1]$).
}\label{fig: prediction_time_donet_ex1}
\centering
\includegraphics[width=1\textwidth]{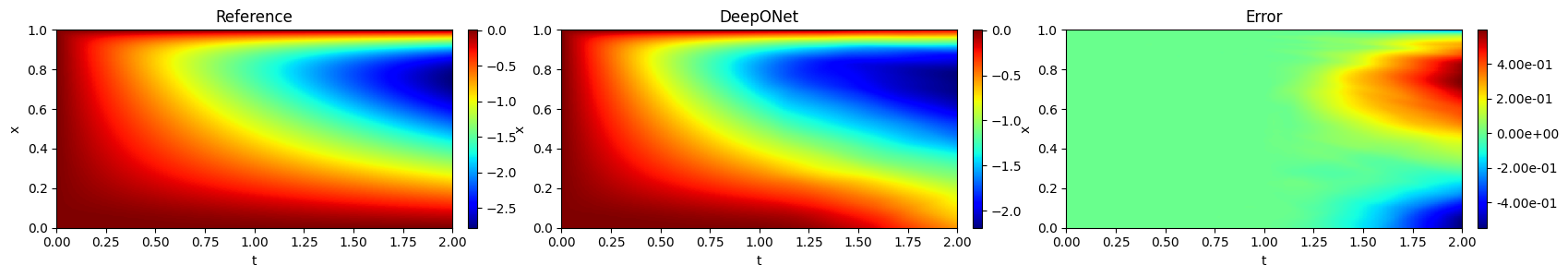}\\
\end{figure}

\begin{figure}[!htbp]
\caption{ Prediction solution of \eqref{eq:diff-react_num} by the NODE-ONet for the learned solution operator with multi-input functions: $\Psi_m^*: \{D(x),f(x)\}\mapsto u(t,x)$ with one random input pair. (Test time frame $t\in [0, 2]$; Training time frame $t\in [0,1]$).
}\label{fig: prediction_time2_ex1}
\centering
\includegraphics[width=1\textwidth]{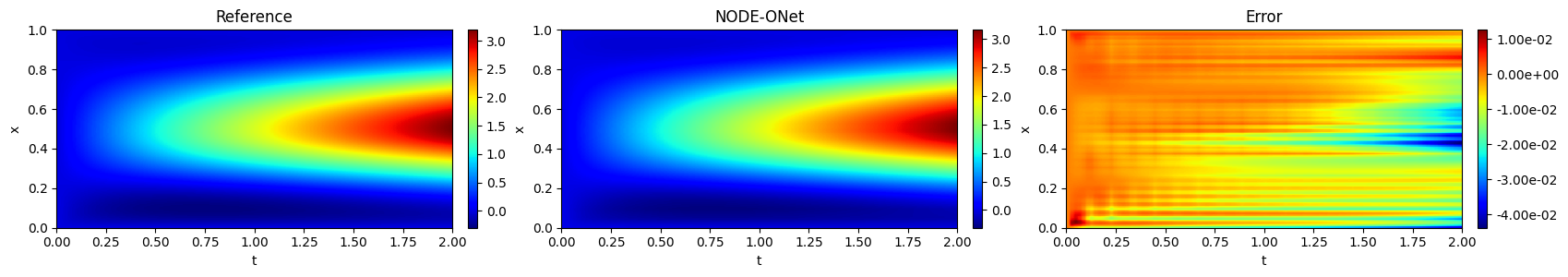}
\end{figure}

\begin{figure}[!htbp]
\caption{ Prediction solution of \eqref{eq:diff-react_num} by the MIONet for  the learned solution operator with multi-input functions: $\Psi_m^*: \{D(x),f(x)\}\mapsto u(t,x)$ with one random input pair. (Test time frame $t\in [0, 2]$; Training time frame $t\in [0,1]$).}\label{fig: prediction_mionet_ex1}
\centering
\includegraphics[width=1\textwidth]{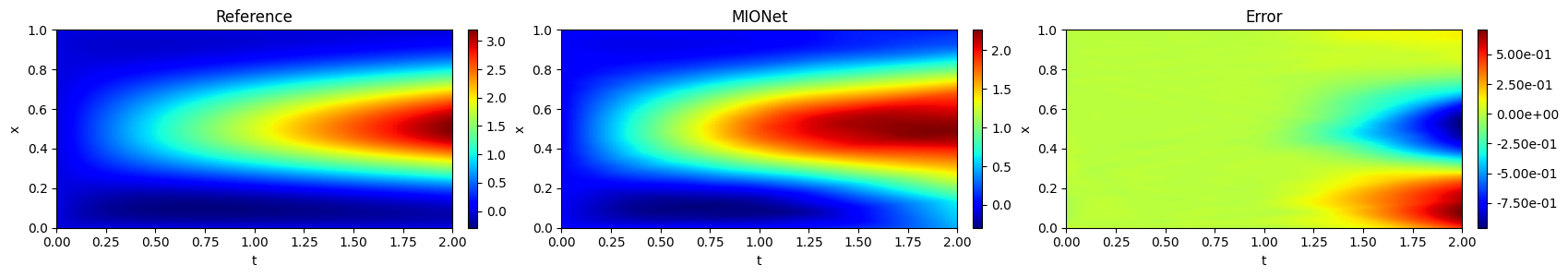}
\end{figure}

\noindent$\bullet$ \textbf{Flexibility of encoder/decoder.} Recall that the NODE-ONet framework does not prescribe specific architectures for the encoder/decoder. In previous simulations, we employ the neural network $\mathcal{N}_{\theta_\alpha}$ to generate $\bm{\alpha}$ in the decoders for comparison purposes. Alternatively, as mentioned in Section \ref{se: architecture}, one can also use some other basis functions for generating $\bm{\alpha}$. To validate, we consider learning the source-to-solution operator $\Psi_f: f\mapsto u$ of \eqref{eq:diff-react_num} by replacing  $\mathcal{N}_{\theta_\alpha}$ with a set of Fourier basis functions, and other settings are kept the same. The test absolute and relative errors are $9.734\times 10^{-4}$ and $1.908\times 10^{-3}$, respectively. The numerical results with one random input $f$ are presented in Figure \ref{fig: fourier_basis_ex1}. All these results demonstrate the flexibility of the NODE-ONet.  
\begin{figure}[!htbp]
\caption{ Test results of  the NODE-ONet with Fourier basis for learning the source-to-solution operator $\Psi_f: f(x)\mapsto u(t,x)$ of \eqref{eq:diff-react_num} with one random input function.
}\label{fig: fourier_basis_ex1}
\centering
\includegraphics[scale=0.3]{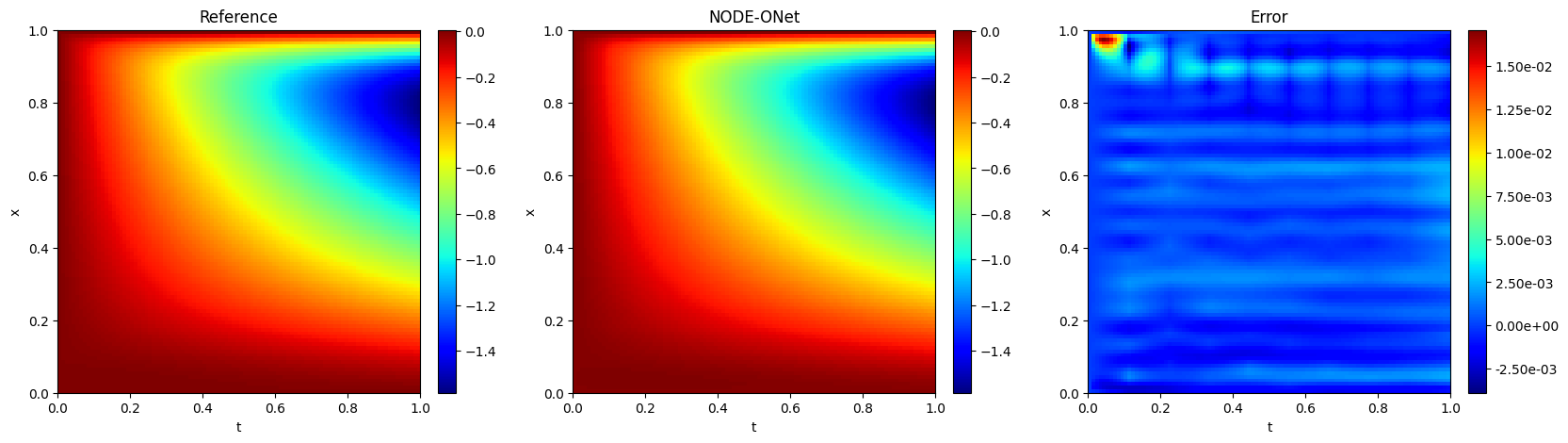}
\end{figure}
\subsection{The Navier-Stokes equation (\ref{eq:2dexample})}

In this section, we test the two-dimensional Navier-Stokes equation \eqref{eq:2dexample} to validate the effectiveness of the NODE-ONet framework. For this purpose, we set the viscosity $\nu=0.001$ and impose the periodic boundary condition on $u$. Following \cite{li2021fourier,lu2022comprehensive}, the forcing term $f$ is assumed to be time-invariant. Then, we are interested in learning the following three operators:
\begin{itemize}
\item the initial value-to-solution operator $\Psi_i: u_0\mapsto u$ with the fixed source term $f(x_1,x_2)=0.1 \sin(2\pi(x_1 + x_2)) + 0.1 \cos(2\pi(x_1 + x_2))$;
\item the source-to-solution operator $\Psi_f: f\mapsto u$ with the fixed initial value $u_0(x_1, x_2)=0.1 \sin(2\pi(x_1 + x_2)) + 0.1 \cos(2\pi(x_1 + x_2))$;
\item the solution operator with multi-input $\Psi_m: \{u_0,f\}\mapsto u$.
\end{itemize}

The training sets for learning the three operators are generated by the public codes provided in \cite{li2021fourier} with the input functions $u_0$ and $f$ sampled from the following Gaussian random fields:
$
u_0\sim \mathcal{GP}(0, 7^{3/2}(-\Delta + 49 I)^{-2.5})$ \text{and} $f\sim \mathcal{GP}(0, 3^{3/2}(-\Delta + 49 I)^{-5}).
$
The neural network $\mathcal{N}_{\theta_\alpha}$ for all the cases is set as an FCNN consisting of 4 hidden layers with $2,000$ neurons per layer and equipped with \texttt{ReLU} activation functions. We  set $A_i^n(t)=\bm{a}_i^1t$ and $P=2,000$ in \eqref{eq:NODE_ns}.  All the NODEs use \texttt{ReLU} activation functions and are solved by the explicit Euler method with $N_t$ time steps.  The output dimension of $\mathcal{N}_{\theta_\alpha}$ and NODEs is set as $d_{\mathcal{U}}=200$.
To train the neural networks, we minimize the loss function  \eqref{eq:loss_dis} with $R(\theta)=\|\theta\|_1$ and $\lambda=10^{-5}$ using an ADAM optimizer followed by certain LBFGS iterations.
The parameters for $\mathcal{N}_{\theta_\alpha}$ and the NODEs  are initialized by the default \texttt{PyTorch} settings. We use the errors defined in \eqref{eq:errors_def} to measure the accuracy of the NODE-ONets for learning the three operators above. Other parameters are summarized in Table \ref{tab:ex2_parameters}. 

\begin{table}[!htbp]
\setlength{\tabcolsep}{0.5em}
\footnotesize
\centering
\caption{Parameter settings for learning the solution operators of \eqref{eq:2dexample}. }
\begin{tabular}[c]{|c||c|c|c|c|c|c|c|c|c|c|c|}
\hline
~ &\multicolumn{1}{|c|}{Training}&\multicolumn{1}{|c|}{Learning} & \multicolumn{1}{|c|}{Training} & \multicolumn{1}{|c|}{Test}
&\multirow{2}{*}{$d_{\mathcal{V}}$} & \multirow{2}{*}{$d_{\mathcal{U}}$}& \multirow{2}{*}{$N_{v}$} \\ 
& \text{epochs} & rate & \text{resolutions} & resolutions 
&&&\\
\hline
$\Psi_i$&  $\begin{tabular}[c]{@{}l@{}}ADAM $5\times10^5$\\ LBFGS 100\end{tabular}$ & $10^{-4}$ & \begin{tabular}[c]{@{}l@{}}$N_x=50^2$\\ $N_t=10$\end{tabular} & $\begin{tabular}[c]{@{}l@{}}$N_x=100^2$\\ $N_t=100$\end{tabular}$ & $50^2$ & 200 &  $\begin{tabular}[c]{@{}l@{}}1000 (training)\\ 200 (test)\end{tabular}$  \\
\hline
$\Psi_f$&  $\begin{tabular}[c]{@{}l@{}}ADAM $5\times10^5$\\ LBFGS 100\end{tabular}$  &  $10^{-4}$ & \begin{tabular}[c]{@{}l@{}}$N_x=50^2$\\ $N_t=10$\end{tabular} & $\begin{tabular}[c]{@{}l@{}}$N_x=100^2$\\ $N_t=100$\end{tabular}$ & $50^2$ & 200 &$\begin{tabular}[c]{@{}l@{}}1000 (training)\\ 200 (test)\end{tabular}$ \\
\hline
$\Psi_m$ &  $\begin{tabular}[c]{@{}l@{}}ADAM $5\times10^5$\\ LBFGS 100\end{tabular}$  & $10^{-4}$ &$\begin{tabular}[c]{@{}l@{}}$N_x=50^2$\\ $N_t=20$\end{tabular}$& $\begin{tabular}[c]{@{}l@{}}$N_x=100^2$\\ $N_t=100$\end{tabular}$ & $50^2$ & 200 & $\begin{tabular}[c]{@{}l@{}}1000 (training)\\ 200 (test)\end{tabular}$\\
\hline

\end{tabular}
\label{tab:ex2_parameters}
\end{table}

\begin{table}[!htpb]
\setlength{\tabcolsep}{0.5em}
\footnotesize
\centering
\caption{Test and prediction accuracy for learning the solution operators of \eqref{eq:2dexample}. }
\begin{tabular}[c]{|c||c|c|c|c|c|c|c|}
\hline
& \text{Training} & Test&Absolute test & Relative test&Absolute test& Relative test\\
& time frame & time frame&error in $[0, 10]$& error in $[0, 10]$&error in $[0, 20]$& error in $[0, 20]$\\
\hline
$\Psi_i$ & $t\in [0,10]$ &  $t\in [0,20]$ &$1.396 \times 10^{-2}$ & $3.053 \times 10^{-2}$ &$5.860 \times 10^{-2}$ & $8.491 \times 10^{-2}$   \\
\hline
$\Psi_f$  & $t\in [0,10]$&  $t\in [0,20]$ &$2.751 \times 10^{-3}$ & $3.180 \times 10^{-2}$ & $ 7.379 \times 10^{-3}$ & $ 7.167 \times 10^{-2}$  \\
\hline
$\Psi_m$ &$t\in [0,10]$& $t\in [0,20]$ & $1.320 \times 10^{-2}$ & $8.827 \times 10^{-2}$ & $1.208\times 10^{-2}$ &$8.857 \times 10^{-2}$ \\
\hline

\end{tabular}
\label{tab:ex2_results}
\end{table}

Numerical results for learning the operators $\Psi_i$, $\Psi_f$, and $\Psi_m$ are presented in Table \ref{tab:ex2_results} and Figures \ref{fig: initial_ex2}$-$\ref{fig: two_inputs_ex2}. The NODE-ONets efficiently learn various solution operators of \eqref{eq:2dexample}, producing highly accurate numerical solutions for $t\in [0,10]$. Furthermore, it maintains satisfactory prediction accuracy for $t\in [10, 20]$ in all cases. These results demonstrate the framework's capability to handle some complex systems in different settings while consistently delivering reliable numerical solutions.

\begin{figure}[!htpb]
\caption{Numerical results for learning the initial value-to-solution operator $\Psi_i: u_0(x)\mapsto u(t, x)$ of \eqref{eq:2dexample} with one random input function. From top to bottom: $t=2, 6, 10$ (test within the training time frame), and $t=12, 20$ (prediction beyond the training time frame).
}\label{fig: initial_ex2}
\centering
\includegraphics[scale=0.35]{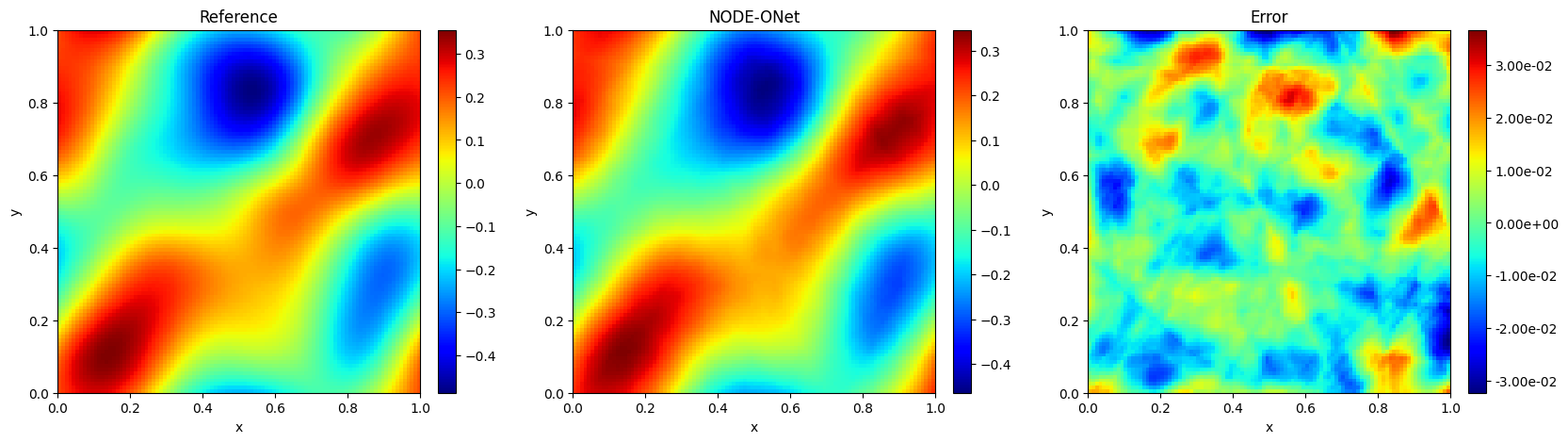}
\includegraphics[scale=0.35]{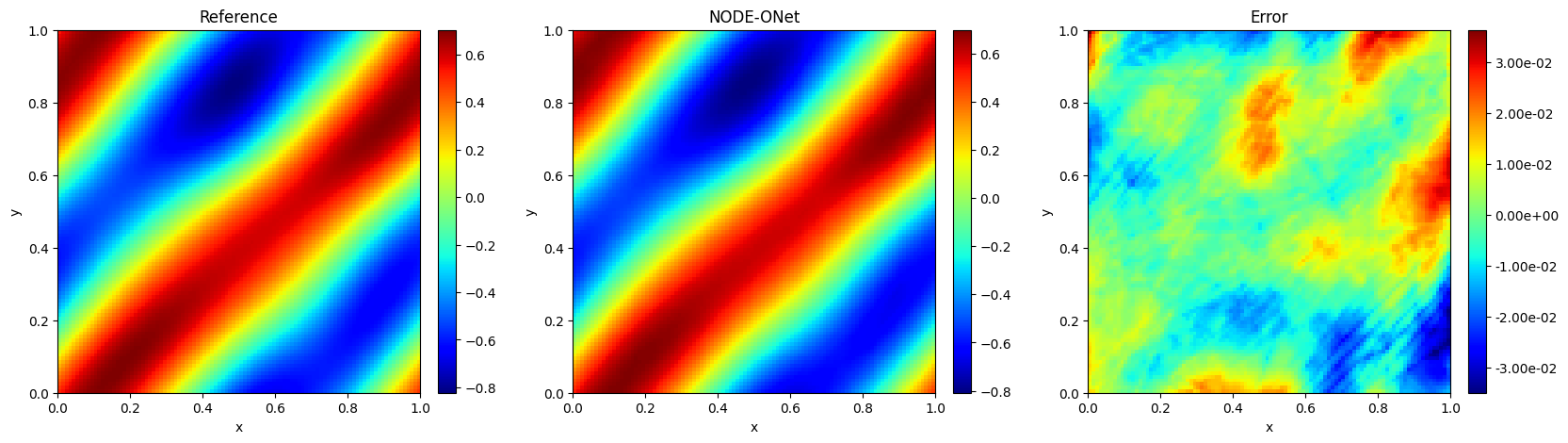}
\includegraphics[scale=0.35]{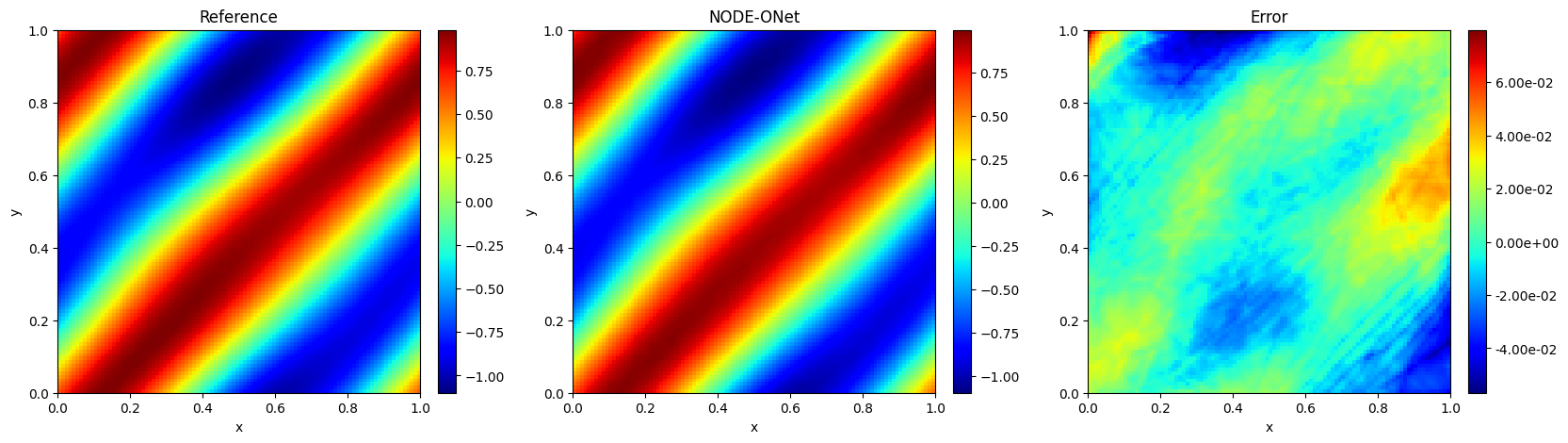}
\includegraphics[scale=0.35]{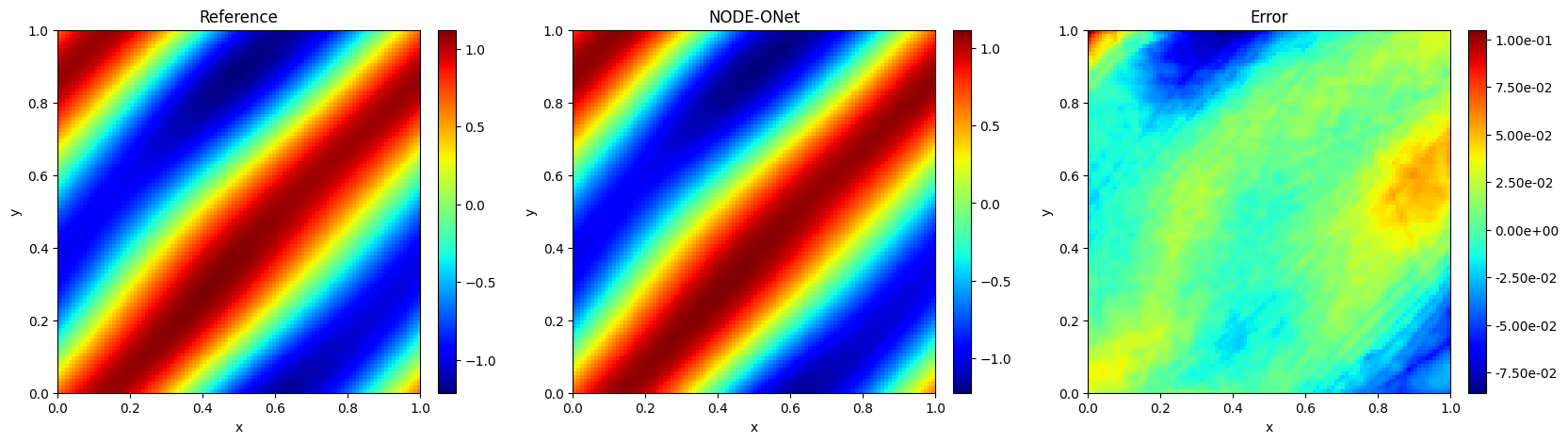}
\includegraphics[scale=0.35]{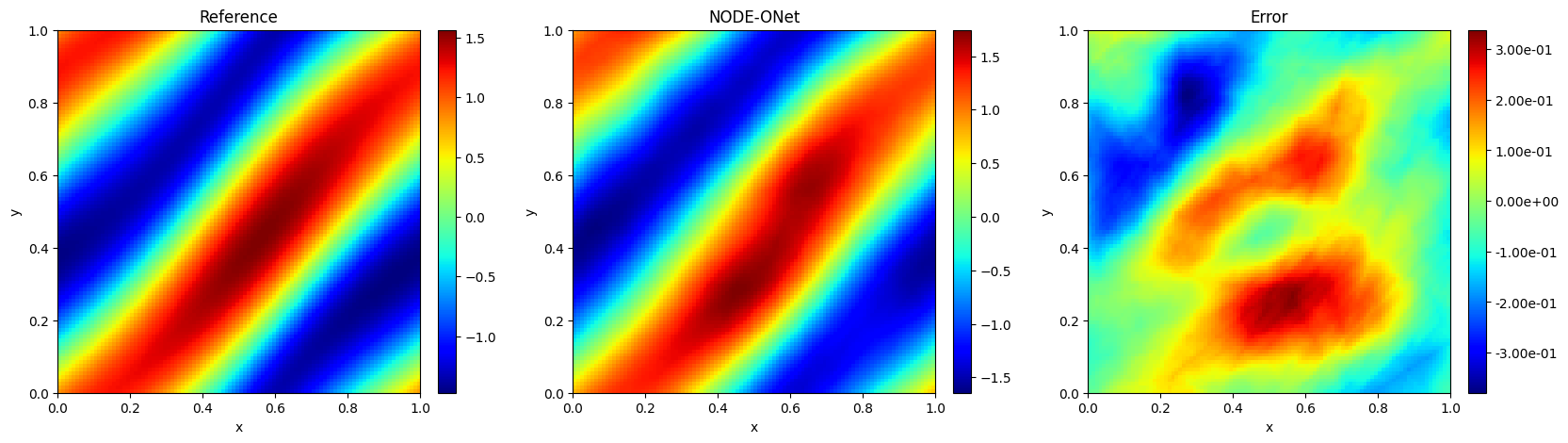}
\end{figure}


\begin{figure}[!htpb]
\caption{Numerical results for learning the source-to-solution operator $\Psi_f: f(x)\mapsto u(t, x)$ of \eqref{eq:2dexample} with one random input function. From top to bottom: $t=2,6,10$ (test within the training time frame), and $t=12, 20$ (prediction beyond the training time frame).
}\label{fig: source_ex2}
\centering
\includegraphics[scale=0.35]{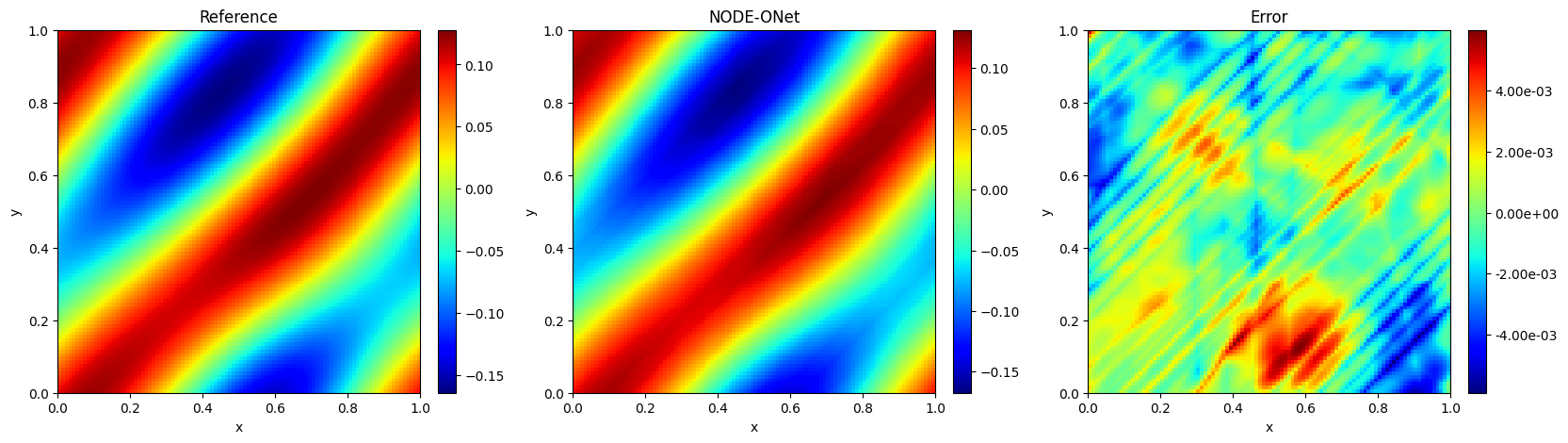}
\includegraphics[scale=0.35]{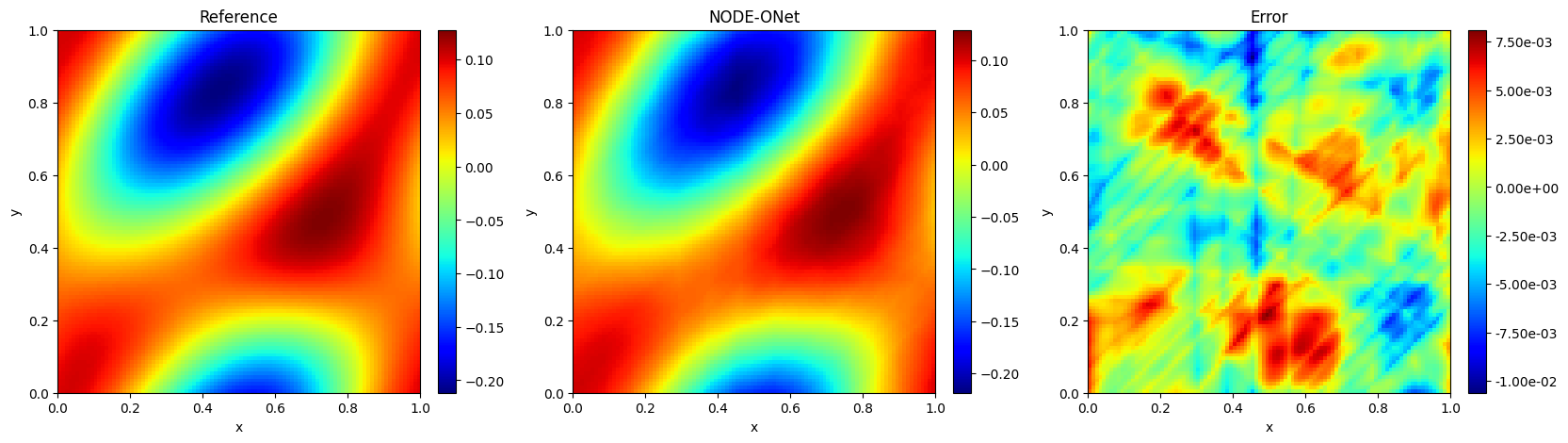}
\includegraphics[scale=0.35]{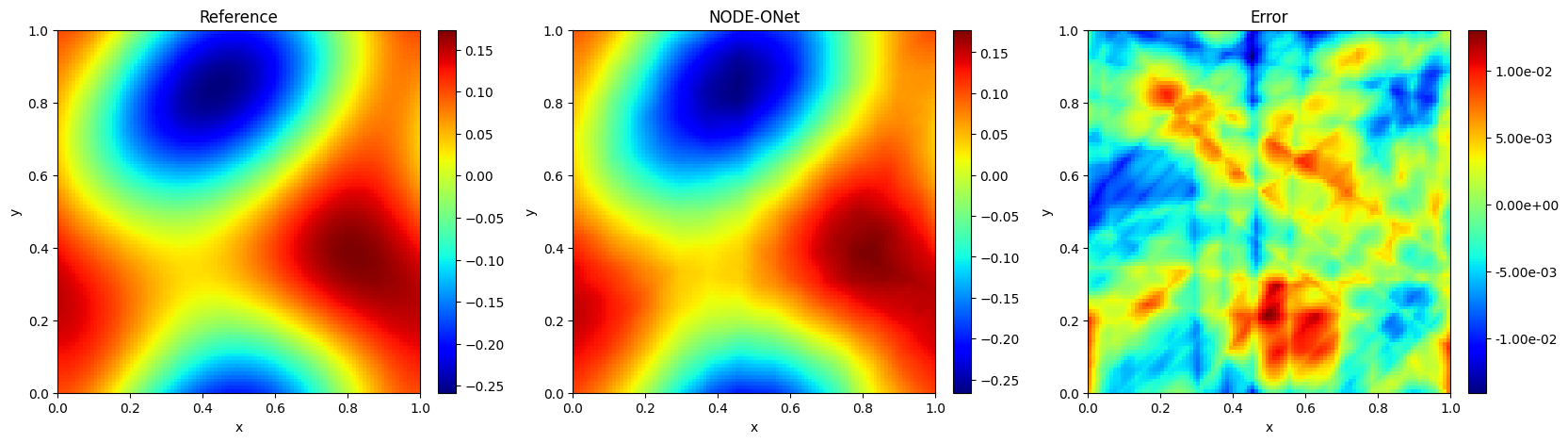}
\includegraphics[scale=0.35]{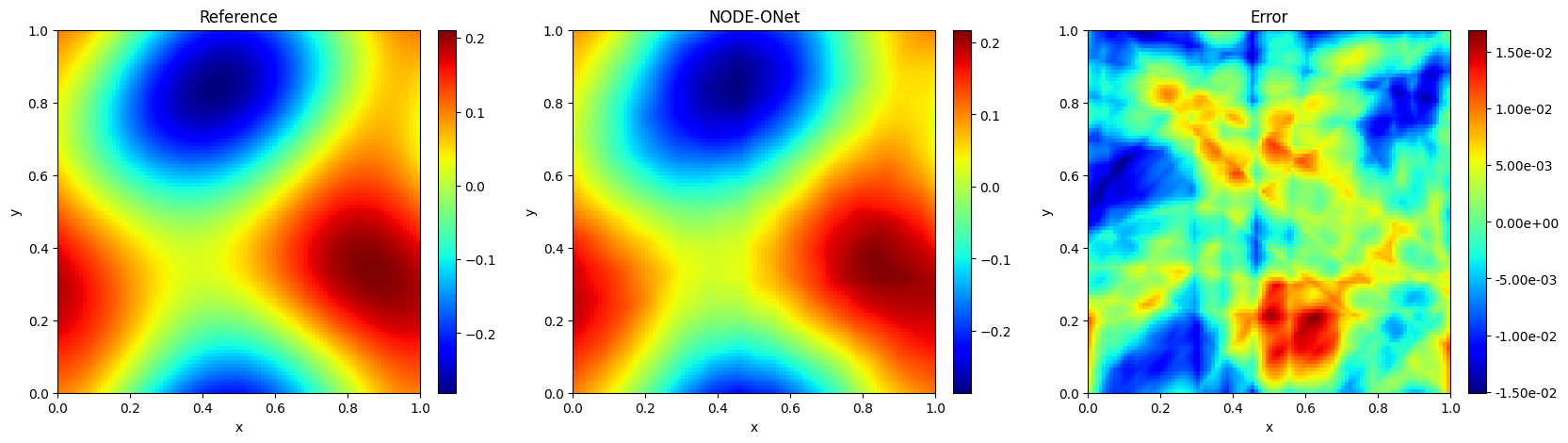}
\includegraphics[scale=0.35]{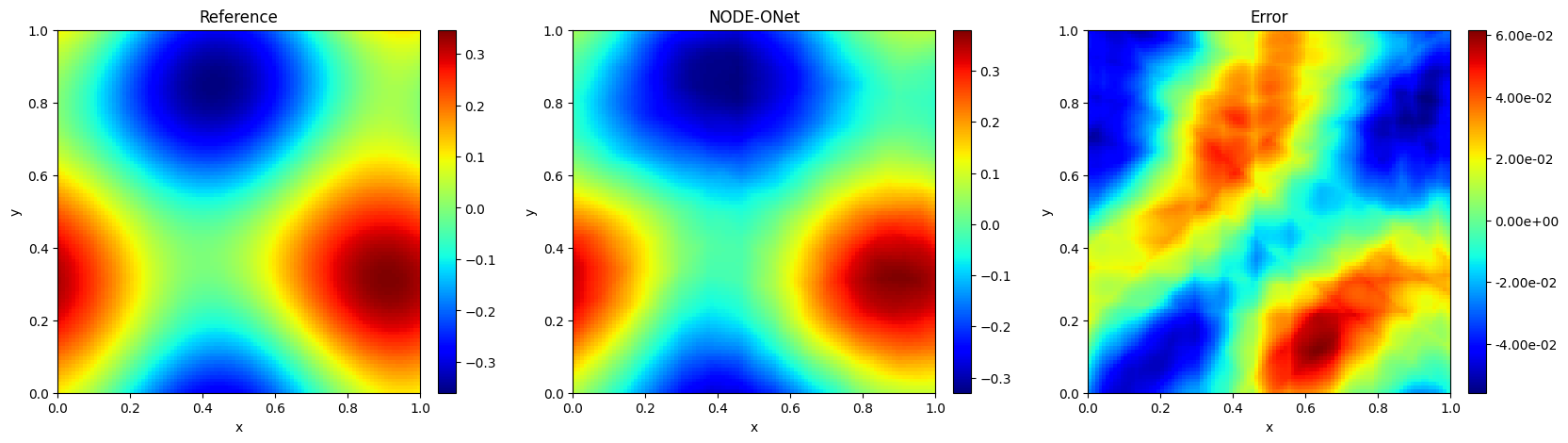}
\end{figure}


\begin{figure}[!htpb]
\caption{ Numerical results for learning the solution operator of \eqref{eq:2dexample} with multi-inputs $\Psi_m: \{u_0(x), f(x)\}\mapsto u(t, x)$ with one random input pair. From top to bottom: $t=2, 6, 10$ (test within the training time frame), and $t=12, 20$ (prediction beyond the training time frame).
}\label{fig: two_inputs_ex2}
\centering
\includegraphics[scale=0.35]{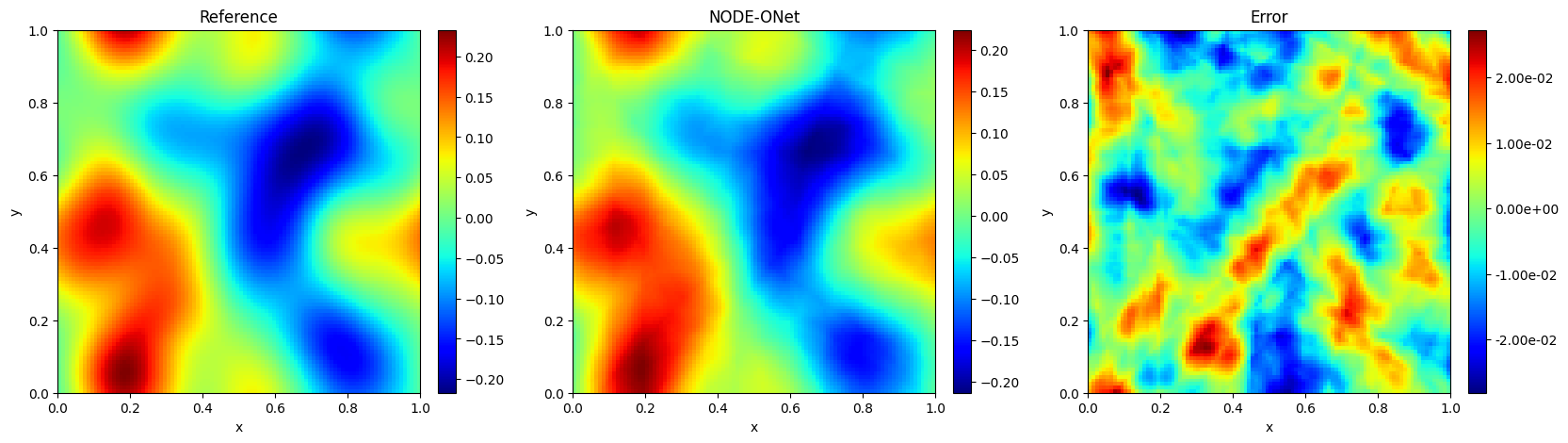}
\includegraphics[scale=0.35]{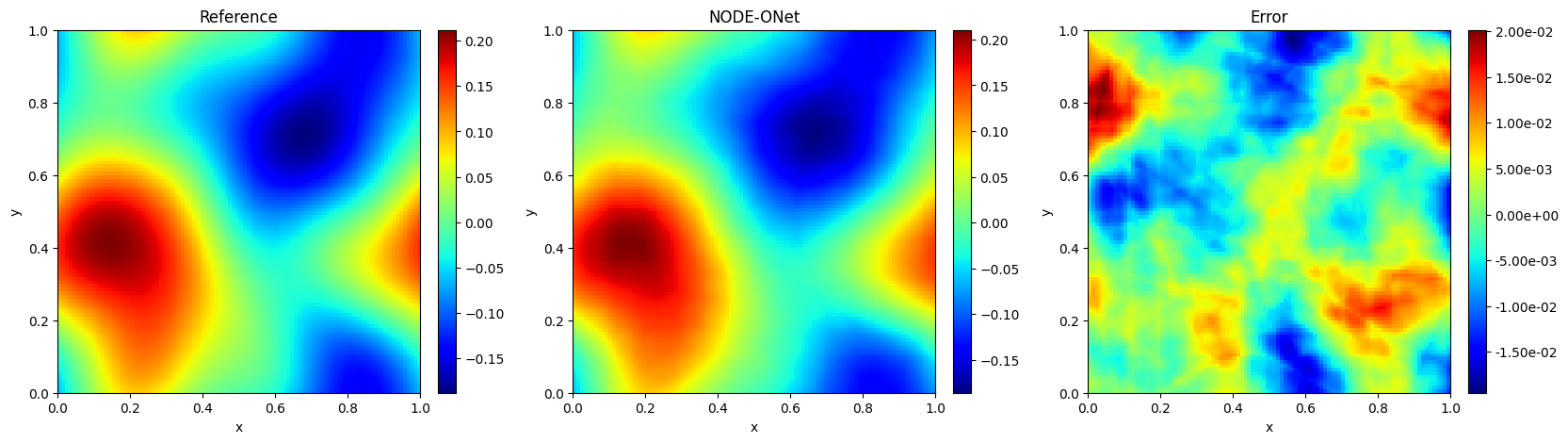}
\includegraphics[scale=0.35]{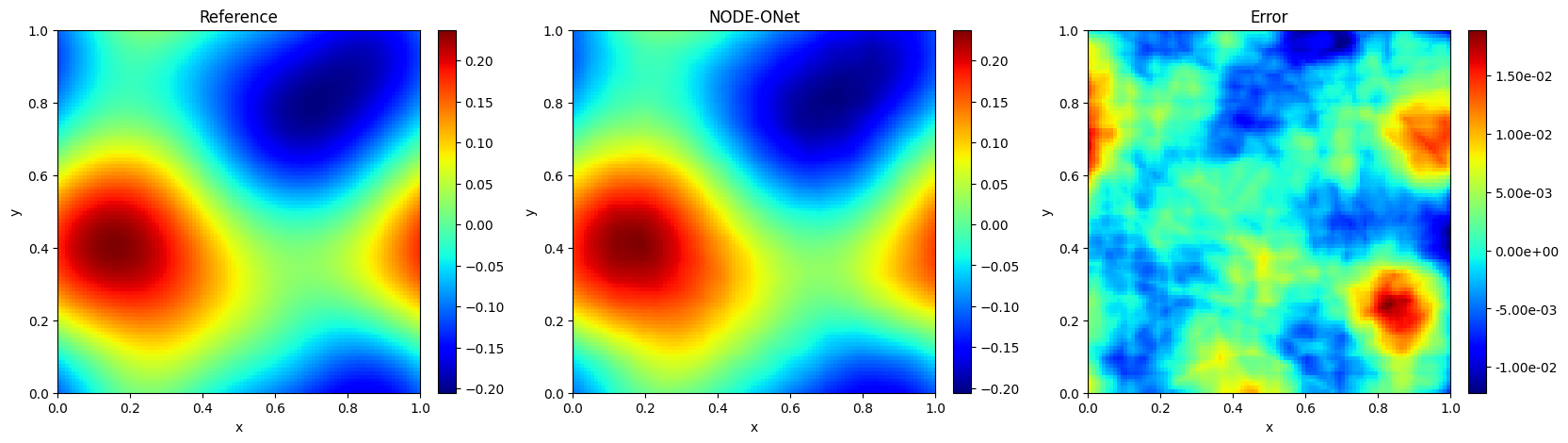}
\includegraphics[scale=0.35]{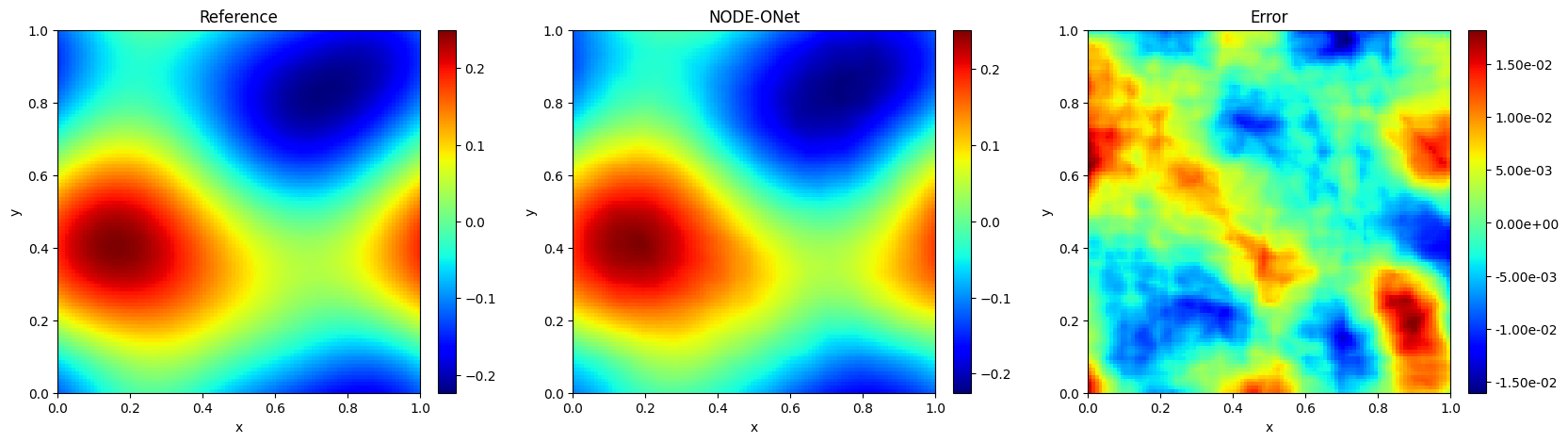}
\includegraphics[scale=0.35]{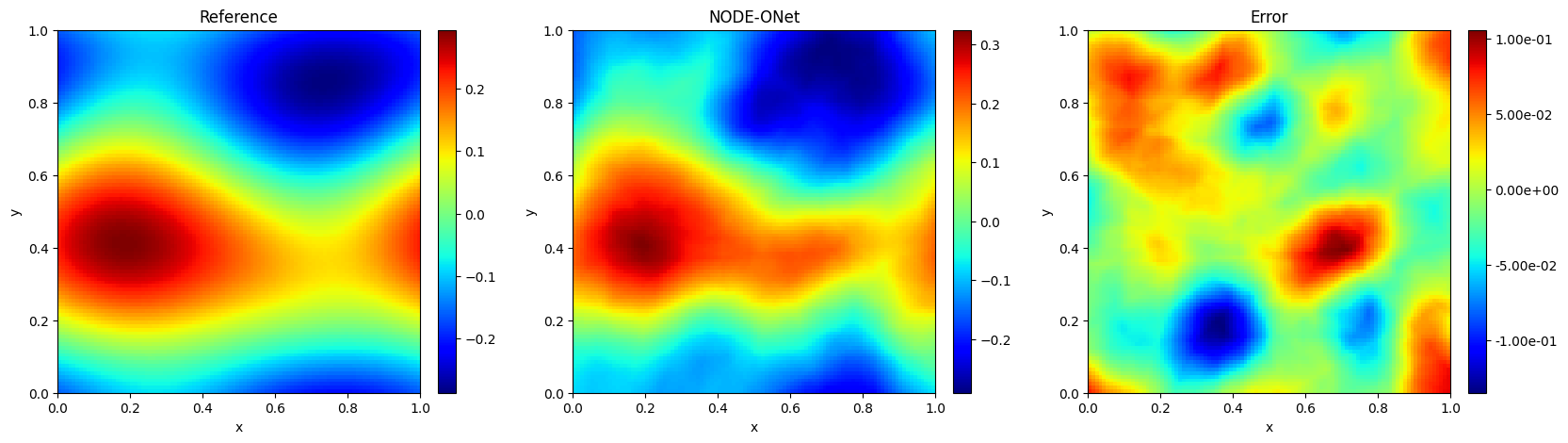}
\end{figure}


\section{Conclusions and Perspectives}\label{se:conclusion}
This paper introduces the deep Neural Ordinary Differential Equation Operator Network (NODE-ONet), a framework for learning solution operators of partial differential equations (PDEs). By integrating neural ODEs (NODEs) within an encoder-decoder architecture, the NODE-ONet framework effectively decouples spatial and temporal variables, aligning with traditional numerical methods for time-dependent PDEs. The core innovation lies in the design of physics-encoded NODEs, which incorporate structural properties of the underlying PDEs (e.g., bilinear couplings, additive source terms) into their architecture. Such well-designed physics-encoded NODEs not only enhance the numerical efficiency and robustness of the resulting NODE-ONets but also enable accurate extrapolation beyond the training temporal domain. 
Our key contributions and findings include:

\begin{itemize}
\item \textbf{Theoretical Foundation:} A general error analysis for encoder-decoder networks is established, providing mathematical insights on operator approximation errors and guiding the design of NODE-ONets.

\item \textbf{Physics-Encoded NODEs:} By enforcing explicit time dependence in trainable parameters and embedding PDE-specific knowledge (e.g., nonlinear dependencies and effects of known PDE parameters), these NODEs achieve superior generalization while maintaining low model complexity.

\item \textbf{Numerical Efficiency:} The NODE-ONets outperform state-of-the-art methods (e.g., DeepONets, MIONet) in terms of numerical accuracy, model complexity, and training cost, especially for learning operators with multi-input functions. 

\item \textbf{Generalization:} Trained encoders/decoders can be transferred to related PDEs without retraining, and predictions remain satisfactory beyond the training time horizon.

\item \textbf{Flexibility:} The framework accommodates various encoders/decoders (e.g., neural networks, Fourier basis) and adapts to both stationary and non-stationary PDEs.
\end{itemize}
Hence, the NODE-ONet framework represents a significant step toward scalable and physics-encoded computational tools for PDEs by combining data-driven learning with mathematical structures.

Our work leaves some important questions, which are beyond the scope of the paper and will be the subject of future investigation.
\begin{itemize}
\item \textbf{Further Error Analysis.} The error analysis presented in Section \ref{se:error_analysis} is mainly devoted to the generic encoder-decoder architecture. It is relevant to analyze the (approximate and generalization) errors for the NODE-ONet framework, which depends intricately on the specific PDE under consideration and is technically involved,  see also Remark \ref{re:analysis}.
\item \textbf{Optimal NODEs.} Note that the physics-encoded NODEs used in our experiments (and also the generic one (\ref{eq:NODE_pe3})) are not unique. For instance, an alternative physics-encoded NODE to \eqref{eq:NODE_m}  can be given by
$$
{\small	\begin{dcases}
\dot{\bm{\psi}}(t) =-P_D\cdot \text{Diag}(D)\cdot P_D^\top\cdot \bm{\psi}+\sum_{i = 1}^P \Big\{W_i\odot\sigma\left(A_i\odot\bm{\psi} +\bm{a}_i^1t+B_i\right)+\mathcal{P}_f\bm{f}\Big\},\\
\bm{\psi}(0) = \bm{0}\in \mathbb{R}^{d_{\mathcal{U}}},
\end{dcases}}
$$
where $\text{Diag}(\cdot): \mathbb{R}^{d_{\mathcal{V}}}\rightarrow \mathbb{R}^{{d_{\mathcal{V}}} \times {d_{\mathcal{V}}}}$ is the diagonal matrix operator, $P_D\in  \mathbb{R}^{{d_{\mathcal{U}}} \times {d_{\mathcal{V}}}}$, and other parameters are the same as those in \eqref{eq:NODE_m}. The above NODE has the same model complexity as that of \eqref{eq:NODE_m} and the resulting NODE-ONet demonstrates comparable numerical performance to that in Table \ref{tab:ex1_compare_minonet}. Hence, it is of great theoretical and practical significance to establish a mathematical principle to determine the optimal physics-encoded NODE for a specific PDE solution operator. 
\item \textbf{Extensions.} 
First, extending the NODE-ONet framework to address optimal control and inverse problems involving PDEs presents a compelling research direction. Such problems typically require solving coupled time-forward and time-backward equations. A central challenge in this extension is, therefore, the design of appropriate NODE architectures capable of simultaneously capturing the evolving dynamics of both equations. Second, while our current focus is on parabolic equations, developing NODE-ONets for hyperbolic equations remains crucial. For this purpose, one may consider some second-order NODEs and the ideas in \cite{ruthotto2020deep} could be useful.

\end{itemize}

\section{Acknowledgment}

Z. Li is supported by the Alexander von Humboldt Professorship program and the Deutsche Forschungsgemeinschaft (DFG, German Research Foundation) under project C07 of the Sonderforschungsbereich/Transregio 154 "Mathematical Modelling, Simulation and Optimization Using the Example of Gas Networks" (project ID: 239904186).

\noindent Y. Song was supported by a Start-Up Grant from Nanyang Technological University. 

\noindent H. Yue was supported by the National Natural
Science Foundation of China (No. 12301399).

\noindent E. Zuazua was funded by the ERC Advanced Grant CoDeFeL (ERC-2022-ADG-101096251) , the grants TED2021-131390B-I00-DasEl of MINECO, and PID2023-146872OB-I00-DyCMaMod of MICIU (Spain); the Alexander von Humboldt Professorship program; the European Union’s Horizon Europe MSCA project ModConFlex (HORIZON-MSCA-2021-DN-01(project 101073558); the Transregio 154 Project “Mathematical Modelling, Simulation and Optimization Using the Example of Gas Networks” of the DFG; the AFOSR 24IOE027 project; and the Madrid Government–UAM Agreement for the Excellence of the University Research Staff in the context of the V PRICIT (Regional Programme of Research and Technological Innovation).

\bibliographystyle{siamplain}

\end{document}